\documentclass{article}

\usepackage{amsmath, amssymb, amsthm}
\usepackage{authblk}
\usepackage{mathrsfs}
\usepackage{geometry}
\usepackage{extarrows}
\usepackage{xcolor}
\usepackage{graphicx}
\usepackage{cite}
\usepackage{enumitem}
\usepackage{titlesec}
\usepackage{hyperref}
\usepackage{array}
\usepackage{subfigure}
\usepackage{booktabs}
\usepackage[title]{appendix}
\usepackage[linesnumbered,ruled,vlined]{algorithm2e}
\SetKwInput{KwInput}{Input}                
\SetKwInput{KwReturn}{Return}

\SetCommentSty{mycommfont}
\usepackage{float}

\newcolumntype{A}{>{\centering\arraybackslash}m{0.12\columnwidth}}
\newcolumntype{B}{>{\centering\arraybackslash}m{0.4\columnwidth}}

\newcommand{\dkl}{D_{\rm KL}}

\def\bE{\mathbb{E}}

\title{Bidirectional Generative Modeling Using Adversarial Gradient Estimation}

\author{Xinwei Shen, Tong Zhang, and Kani Chen}
\affil{The Hong Kong University of Science and Technology\\
\texttt{\{xshenal,tongzhang,makchen\}@ust.hk}}
\date{}

\DeclareMathOperator*{\argmin}{argmin}

\theoremstyle{plain}
\newtheorem{thm}{Theorem}
\newtheorem{lemma}{Lemma}

\theoremstyle{definition}

\begin{document}
\maketitle

\begin{abstract}
\noindent
This paper considers the general $f$-divergence formulation of bidirectional generative modeling, which includes VAE and BiGAN as special cases. We present a new optimization method for this formulation, where the gradient is computed using an adversarially learned discriminator. In our framework, we show that different divergences induce similar algorithms in terms of gradient evaluation, except with different scaling. Therefore this paper gives a general recipe for a class of principled $f$-divergence based generative modeling methods. Theoretical justifications and extensive empirical studies are provided to demonstrate the advantage of our approach over existing methods. 
\end{abstract}

\section{Introduction}

Deep generative modeling has aroused a lot of interest as a method for data generation and representation learning.
Consider the observed real data $X$ from an unknown distribution $p_r$ on $\mathcal{X}\subseteq\mathbb{R}^d$ and the latent variable $Z$ with a known prior $p_z$ on $\mathcal{Z}\subseteq\mathbb{R}^k$. In unidirectional data generation, we are interested in learning a transformation $G:\mathcal{Z}\times\mathcal{E}\to\mathcal{X}$ so that the distribution of the transformed variable $G(Z,\epsilon)$ becomes close to $p_r$, where $\epsilon\in\mathcal{E}$ is the source of randomness with a specified distribution $p_\epsilon$ and $G$ is referred to as a \textit{generator}. In many applications, bidirectional generative modeling is favored due to the ability to learn representations, where we additionally learn a transformation $E:\mathcal{X}\times\mathcal{E}\to\mathcal{Z}$, known as an \textit{encoder}. 

The principled formulation of bidirectional generation is to match the distributions of two data-latent pairs $(X,E(X,\epsilon))$ and $(G(Z,\epsilon),Z)$.
Classical methods including Variational Autoencoder (VAE) \cite{kingma2013auto} and Bidirectional Generative Adversarial Network (BiGAN) \cite{Donahue2017AdversarialFL,Dumoulin2017AdversariallyLI} turn out to handle this task using one specific distance measure as the objective. In this paper, we generally consider the $f$-divergence which is a natural and broad class of distance measures. 

For optimization, both VAE and BiGAN are limited to specific divergences and assumptions for the encoder and generator distributions, and hence do not apply in the general formulation. $f$-GAN \cite{nowozin2016f} extends GAN \cite{goodfellow2014generative} to $f$-divergence and can be applied in the formulation here. Like GAN, $f$-GAN introduces a discriminator to distinguish between two data-latent pairs. However, we find that limited by the minimax formulation, the discriminator loss of $f$-GAN tends to behave poorly in both statistical efficiency and training stability. 
Other methods \cite{poole2016improved,mohamed2016learning,mescheder2017adversarial,srivastava2017veegan,chen2017symmetric} propose to estimate the objective of $f$-divergence based on density ratio estimation and adopt adversarial training. However the consequent algorithms are heuristic without guarantee for convergence and cannot be reasonably applied to bidirectional cases, which is further  discussed in Section \ref{sec:relate_div_min}.

This paper proposes a new optimization method for this formulation.
We present a theorem to evaluate the gradient of the divergence with respect to the generator and encoder parameters, which generally applies to various $f$-divergences with the only difference being the scaling. Then we propose an efficient gradient estimator using a discriminator learned with nonlinear Logistic regression. Based on the theory and estimation, we obtain a family of algorithms, and hence gives a general recipe for a class of principled $f$-divergence based generative modeling methods. We further propose an applicable technique and obtain an algorithm which locally minimizes several divergences simultaneously with a lower variance and stable gradients. 


We highlight our main contributions as follows:
\begin{itemize}\vspace{-0.05in}
\setlength{\itemsep}{1pt}
\setlength{\parskip}{0pt}
\item We derive the formula to evaluate the gradient of a general $f$-divergence wrt. model parameters, which enables a principled family of algorithms for $f$-divergence based bidirectional generative modeling. 
\item We give theoretical insights on the $f$-divergence formulation in unidirectional generation, mode coverage and cycle consistency, and present a unified view of VAEs and GANs.
\item We conduct extensive empirical results on synthetic and real datasets to demonstrate: (i) the effectiveness of our optimization method in $f$-divergence minimization, and (ii) the advantages of our learned bidirectional generative models in mode coverage, realistic generation and high-level semantic representation.

\end{itemize}\vspace{-0.3cm}

\medskip
\noindent
\textbf{Notation} \ \ 
Throughout the paper, all distributions are assumed to be absolutely continuous with respect to Lebesgue measure unless indicated otherwise. Let $p_g(x|z)$ and $p_e(z|x)$ be the conditional distributions induced by $G$ and $E$. 
For a scalar function $h(x,y)$, let $\nabla_x h(x,y)$ denote its gradient with respect to $x$. For a vector function $g(x,y)$, let $\nabla_x g(x,y)$ denote its Jacobi matrix with respect to $x$. 


\section{Bidirectional generative modeling}

\subsection{General formulation}

The goal of bidirectional generative modeling is to match the two joint distributions
\begin{equation}\label{eq:obj_bigen}
	\min_{G,E}\ \mathcal{L}(p_e(x,z),p_g(x,z)),
\end{equation}
where $\mathcal{L}$ is any chosen distance measure between two probability distributions: $p_e(x,z)=p_r(x)p_e(z|x)$ is the encoder joint distribution and $p_g(x,z)=p_z(z)p_g(x|z)$ is the generator joint distribution. 
Let us consider a case where $\mathcal{L}$ is an $f$-divergence. Formally, given two density functions $p_e(x,z)$ and $p_g(x,z)$, abbreviated as $p_e$ and $p_g$ for simplicity, the $f$-divergence is defined by 
\begin{equation}\label{eq:f_div}
D_f(p_e,p_g)=\mathbb{E}_{p_g(x,z)}[f\left(r(x,z)\right)]=\mathbb{E}_{p_e(x,z)}\hspace{-3pt}\left[\tilde{f}\left(\frac{1}{r(x,z)}\right)\right]
\end{equation}
where $f:\mathbb{R}_+\to\mathbb{R}$ is a convex, lower-semicontinuous function satisfying $f(1) = 0$, $\tilde{f}(r)=rf(1/r)$ is a notation for convenience, and $r(x,z)=p_e(x,z)/p_g(x,z)$. Here we focus on a special case where $f$ is twice  continuously differentiable and strongly convex so that the second order derivative of $f$, denoted by $f''$, is always positive. The commonly used special cases are listed in Table \ref{tab:f-div}.

\begin{table}[b]
\centering\small
\caption{List of $f$-divergences: KL, reverse KL divergence, Jensen-Shannon divergence$\times2$, and squared Hellinger distance. When the divergence is symmetric, $f=\tilde{f}$, e.g., JS and $H^2$.}\label{tab:f-div}
\vskip 0.1in
\begin{tabular}{c|cccc}
\toprule
Name & $f(r)$ & $rf''(r)$ & $\tilde{f}(r)$ & $r\tilde{f}''(r)$  \\\midrule
KL & $r\log r$ & $1$ &  $-\log r$ & $1/r$ \\
RevKL & $-\log r$ & $1/r$ & $r\log r$ & $1$ \\
2JS & $-(r+1)\log\frac{1+r}{2}+r\log r$ & $\frac{1}{1+r}$ \\
$H^2$ & $(\sqrt{r}-1)^2$ & $\frac{1}{2\sqrt{r}}$ \\
\bottomrule
\end{tabular}
\end{table}


We parametrize the two transformations using deep neural networks and write $G_\theta(z,\epsilon)$ and $E_\phi(x,\epsilon)$. Examples of transformations include additive Gaussian $E(x,\epsilon)=E_{\phi_0}(x)+\phi_1\cdot\epsilon$ where $\epsilon$ follows a Gaussian distribution, or a black-box transformation $G_\theta(z,\epsilon)$ where $\epsilon$ is fed into the input or intermediate layers of the network $G$, leading to an implicit conditional distribution $p_g(x|z)$. 
A detailed discussion on the choice of transformations is given in Appendix \ref{app:choice_dist}. 
Finally our goal is to minimize the objective
\begin{equation}\label{eq:obj_fdiv}
L(\theta,\phi)=D_f(p_e,p_g),
\end{equation}
with respect to the parameters $\theta$ and $\phi$.

\subsection{Advantages}\label{sec:cc}

In this section, we discuss the advantages of the above formulation that minimizes the bidirectional $f$-divergence, especially the KL divergence which is the main choice in this paper. Our justifications cover three aspects of interest.

\medskip
\noindent
\textbf{Unidirectional generation} \ \ 
Decompose the joint KL as 
\begin{align}\label{eq:kl_decom_uni}
\begin{split}
\dkl(p_e(x,z),p_g(x,z))=\dkl(p_r(x),p_g(x))+\mathbb{E}_{x\sim p_r(x)}[\dkl(p_e(z|x),p_g(z|x))],
\end{split}
\end{align}
where we have marginal densities $p_g(x)=\mathbb{E}_{z\sim p_z}[p_g(x|z)]$, $p_e(x)=\mathbb{E}_{x\sim p_r(x)}[p_e(z|x)]$, and the posterior $p_g(z|x)=p_g(x,z)/p_g(x)$. We prove the equivalence in Appendix \ref{app:kl_decom}. Because KL is always non-negative, we know that by minimizing $\dkl(p_e,p_g)$, we minimize an upper bound of $\dkl(p_r(x),p_g(x))$, which is a standard objective for unidirectional generative modeling equivalent to maximum likelihood. By symmetry, same results also hold for $\dkl(p_z(z),p_e(z))$.
Therefore, this bidirectional formulation can approximately achieve the goal of unidirectional generation, while the performance depends on how well the bidirectional model can match the two conditional distributions $p_e(z|x)$ and $p_g(z|x)$, or ensure the consistency between the two transformations. 

\medskip
\noindent
\textbf{Mode coverage} \ \ 
Write the joint KL as
\begin{eqnarray*}
\dkl(p_e(x,z),p_g(x,z))=\mathbb{E}_{p_e(x,z)}\left[\log\frac{p_r(x)p_e(z|x)}{p_g(x)p_g(z|x)}\right].
\end{eqnarray*}
We see that it imposes a heavy penalty when $p_g(x)\approx0$ while $p_r(x)>0$, which is a case of mode dropping. In contrast, other divergences like JS, reverse KL divergence or Square Hellinger distance do not have this property. This is consistent with the commonly known conjecture in unidirectional generation that KL has an advantage in diminishing mode collapse. However this was not well verified in practice \cite{nowozin2016f}, partially due to lack of effective optimization. In this work equipped with the proposed optimization approach introduced in Section \ref{sec:age}, we are able to provide more convincing evidence on this.


\medskip
\noindent
\textbf{Cycle consistency} \ \ 
Another important issue in bidirectional generative modeling is the \emph{cycle consistency}, roughly meaning that the inferred latent variable $E(x)$\footnote{For simplicity we omit the randomness $\epsilon$ in the notations of encoder $E$ and generator $G$.} from data $x$ can generate a data $G(E(x))$ that is very close to $x$. When using stochastic transformations, we define the cycle consistency from a probabilistic view as the expected reconstruction log-likelihood:
\begin{equation}\label{eq:cc}
L_{\text{CC}}=-\mathbb{E}_{x\sim p_r(x)}\mathbb{E}_{z\sim p_e(z|x)}[\log p_g(x|z)].
\end{equation}
We would like to minimize the above quantity, that is, to reconstruct $x$ with a high probability, in order to ensure cycle consistency. Previous methods ensure cycle consistency using an explicit reconstruction error term, i.e., $\|G(E(x))-x\|$, and commonly used norms include $L_1$ and $L_2$ \cite{zhu2017unpaired,li2017alice} which can be regarded as special cases of (\ref{eq:cc}) with the generator being a Laplace or a Gaussian distribution.
Write the joint KL equivalently as
\begin{eqnarray}\label{eq:kl_decom_cc}
\begin{aligned}
\dkl(p_e,p_g)=-\mathbb{E}_{x\sim p_r}\mathbb{E}_{z\sim p_e(z|x)}[\log p_g(x|z)]+\mathbb{E}_{x\sim p_r}\left[\dkl(p_e(z|x),p_z(z))\right]+\mathbb{E}_{x\sim p_r}[\log p_r(x)],
\end{aligned}
\end{eqnarray}
which is proved in Appendix \ref{app:kl_decom}. Note that the third term on the right-hand side of (\ref{eq:kl_decom_cc}) is free of parameters. Hence our formulation equivalently minimizes an upper bound of $L_{\text{CC}}$ and thus ensures cycle consistency. 



\section{$f$-divergence minimization}\label{sec:age}



\subsection{Adversarial gradient estimation} 
We formally propose an optimization approach for the above formulation, leading to a general recipe for principled $f$-divergence based generative modeling. 
From the following theorem, we can evaluate the gradients of the $f$-divergence in objective (\ref{eq:obj_fdiv}) with respect to the parameters. The proof can be found in Appendix \ref{app:thm_proof}.


\begin{thm}\label{thm:grad}
Let $\mathcal D(x,z)=\log(p_e(x,z)/p_g(x,z))$. Then we have 
\begin{align}
\begin{split}\label{eq:grad}
\nabla_\theta L(\theta,\phi)&=-\mathbb{E}_{z\sim p_z(z),\epsilon\sim p_\epsilon}\left[s_\theta(G_\theta(z,\epsilon),z)\nabla_x \mathcal D(G_\theta(z,\epsilon),z)^\top \nabla_\theta G_\theta(z,\epsilon)\right],\\
\nabla_\phi L(\theta,\phi)&= \mathbb{E}_{x\sim p_r(x),\epsilon\sim p_\epsilon}\left[s_\phi(x,E_\phi(x,\epsilon))\nabla_z \mathcal D(x,E_\phi(x,\epsilon))^\top \nabla_\phi E_\phi(x,\epsilon)\right],
\end{split}
\end{align}
where $s_\theta(x,z)=\tilde{f}''\left(1/r(x,z)\right)/r(x,z)$ and $s_\phi(x,z)=f''(r(x,z))r(x,z)$ are scaling factors.
\end{thm}

This theorem presents a general formula to evaluate gradients that applies to various divergences with the only difference being the scaling, which unifies the treatment of $f$-divergence based generative modeling. 

Notice that the gradients in (\ref{eq:grad}) depend on the unknown densities $p_e$ and $p_g$ and thus cannot be obtained from data. We use a discriminator to estimate them. 
Let $D(x,z)$ be the solution to the empirical Logistic regression problem that distinguishes between the data-latent pairs from $p_e$ and $p_g$:\footnote{Logistic regression used in this paper is the nonlinear one. Note that (\ref{eq:obj_logistic}) is equivalent to the loss used in many papers: $\max_d[\sum_{(x,z)\in S_e}\log(d(x,z))/|S_e|+\sum_{(x,z)\in S_g}\log(1-d(x,z))/|S_g|]$ where $d=1/(1+e^{-D})\in(0,1)$.}
\begin{equation}\label{eq:obj_logistic}
\min_{D'}\left[\frac{1}{|S_e|}\sum_{(x,z)\in S_e}\log(1+e^{-D'(x,z)}) + \frac{1}{|S_g|}\sum_{(x,z)\in S_g}\log(1+e^{D'(x,z)})\right]
\end{equation}
where $S_e$ and $S_g$ are finite samples from $p_e(x,z)$ and $p_g(x,z)$ respectively. When the number of samples is sufficiently large, the statistical consistency theory of Logistic regression \cite{zhang2004statistical} indicates that $D(x,z)\approx\mathcal D(x,z)$.

Replacing $\mathcal D$ and $r$ in the gradients (\ref{eq:grad}) with $D$ and $\hat{r}=\exp(D)$, we obtain the maximum likelihood estimator (MLE) for the gradients. We then optimize the objective using stochastic gradient descent (SGD) and end up with a practical implementation. The convergence of the procedure follows naturally from the consistency of the estimation and the convergence results of SGD. Since the proposed approach involves an adversarially learned discriminator, we call it \textit{Adversarial Gradient EStimation (AGES)}. We adopt early stopping in training $D$ to avoid overtrained extreme discriminators. We summarize the procedure of bidirectional generative modeling using AGES in Algorithm \ref{algo}.

In addition, the technique introduced in Theorem \ref{thm:grad} is not limited to bidirectional generation, but can be generally applied to other tasks involving $f$-divergence optimization such as unidirectional generative modeling, mutual information optimization or $f$-divergence as a regularization term (e.g., in WAE \cite{Tolstikhin2017WassersteinA} or VAE-based disentanglement methods \cite{kim2018disentangling}). In Appendix \ref{app:uni_age} we present the gradient formula and estimation in unidirectional generation.

{\centering
\begin{minipage}{.75\linewidth}
\vskip 0.1in
\begin{algorithm}[H]
\DontPrintSemicolon
\KwInput{training set, $f$-divergence, initial parameters $\theta,\phi,\psi$, batch-size $n$}
\While{not convergence}{
\For{multiple steps}{
Sample $\{x_1,\ldots,x_n\}$ from the training set\\
Sample $\{z_1,\ldots,z_n\}$ from the prior $p_z(z)$\\
Sample $\{\epsilon_1,\ldots,\epsilon_n\}$ and $\{\epsilon'_1,\ldots,\epsilon'_n\}$ from $p_\epsilon$ and $p_{\epsilon'}$\\
Update $\psi$ by descending the stochastic gradient:
$\frac{1}{n}\sum_{i=1}^n \nabla_\psi\left[\log(1+e^{-D_\psi(x_i,E_\phi(x_i,\epsilon_i))})+\log(1+e^{D_\psi(G_\theta(z_i,\epsilon'_i),z_i)})\right]$
}
Sample $\{x_1,\ldots,x_n\}$, $\{\epsilon_1,\ldots,\epsilon_n\}$, $\{z_1,\ldots,z_n\}$, and $\{\epsilon'_1,\ldots,\epsilon'_n\}$ as above\\
Compute $\theta$-gradient:
$-\frac{1}{n}\sum_{i=1}^n s_\theta(G_\theta(z_i,\epsilon'_i),z_i)\nabla_\theta D_\psi(G_\theta(z_i,\epsilon'_i),z_i)$\\
Compute $\phi$-gradient:
$\frac{1}{n}\sum_{i=1}^n s_\phi(x_i,E_\phi(x_i,\epsilon_i)) \nabla_\phi D_\psi(x_i,E_\phi(x_i,\epsilon_i))$\\
Update parameters $\theta,\phi$ using the gradients
}
\KwReturn{$\theta,\phi$}
\caption{Bidirectional Generative Modeling using AGES}
\label{algo}
\end{algorithm}
\end{minipage}
\par
}

\subsection{Scaling clipping}\label{sec:sc}
In this section, we introduce a technique to reduce the variance and stabilize training of AGES algorithms for various divergences, and further obtain a modified algorithm that is more applicable on real datasets. 
From Table \ref{tab:f-div} we know that for all commonly used $f$-divergences, one or both of the scaling factors in (\ref{eq:grad}) are unbounded above or can infinitely approach $0$, which may lead to the exploding or vanishing gradient problem especially on real datasets. To address this, we propose to clip the scaling factors of each divergence into a bounded positive range.

From the definition $\tilde{f}(r)=rf(1/r)$ we know $s_\theta(x,z)=s_\phi(x,z) r(x,z)$. Consider the nearly optimal case where the two joint distributions $p_e(x,z)$ and $ p_g(x,z)$ approximately match, and hence $r(x,z)\approx1$. Because $f''$ and $\tilde{f}''$ are positive and continuous, we have $s_\theta(x,z)\approx s_\phi(x,z)\approx f''(1)$ which is a positive constant. Therefore we propose to clip the density ratio $r$ into a bounded range containing its optimal value 1. Then the consequent scaling factors globally fall into a bounded positive range containing $f''(1)$. We call this technique \textit{scaling clipping (SC)}. In this way the corresponding gradient estimator has a smaller variance and we obtain a modified family of algorithms for different divergences with stable gradients. 

Motivated by the local property that $s_\theta(x,z)\approx s_\phi(x,z)\approx$ constant, we consider the extreme case of scaling clipping where we set $s_\theta=s_\phi=1$. By this means we obtain an algorithm which is locally equivalent to simultaneously minimizing several divergences, i.e., all the $f$-divergences with strongly convex $f$. We hence call it \textit{AGES-ALL}. As scaling clipping, AGES-ALL is globally bounded and thus has a smaller variance and alleviates the vanishing or exploding gradient problem.

\subsection{Comparison with $f$-GAN}\label{sec:fgan}

$f$-GAN \cite{nowozin2016f} extends GAN to general $f$-divergences and makes use of their variational representation for optimization. One can also derive a bidirectional version of $f$-GAN by augmenting the variational function in $f$-GAN to a joint version with both $x$ and $z$ as input. The variational function serves as the discriminator in our method in the sense that they both estimate a function of the density ratio $p_e/p_g$. Mathematically, bidirectional $f$-GAN solves the following minimax optimization problem:
\begin{equation*}
\min_{G,E}\max_D \left\{\mathbb{E}_{p_e(x,z)}[a_f(D(x,z))]+\mathbb{E}_{p_g(x,z)}[-f^*(a_f(D(x,z)))]\right\}
\end{equation*}
where $f^*(t)=\sup_{r\in{\rm dom}_f}\{rt-f(r)\}$ is the conjugate function of $f$ and $a_f$ is an output activation function specific to the $f$-divergence used.
However, $f$-GAN generally obtains different training objectives from AGES given the same distance measure, especially the loss of the discriminator.

For instance, Table \ref{tab:kl_obj} lists the training objectives of $f$-GAN and AGES for KL. Note that for comparison we present AGES in the ``GAN form" where we separately write the objectives for the three agents -- discriminator $D$, encoder $E$ and generator $G$, and in each objective we should ignore the dependence of it on the other two agents according to Theorem \ref{thm:grad}. We notice that using KL as the objective, AGES differs from $f$-GAN only in the $D$ loss. Viewing the role of $D$ as the density ratio estimator, we know that AGES obtains the MLE with higher efficiency than $f$-GAN. Moreover, the exponential in the $D$ loss given by $f$-GAN may cause instability during training. We conduct experiments to verify the advantage of AGES against $f$-GAN.

\begin{table}[b]
\centering\small
\caption{Training objectives of $f$-GAN and AGES for KL.}
\vskip 0.1in
\begin{tabular}{c|c}
\toprule
Method &{\centering Objectives} \\
\midrule
AGES & \small{\begin{tabular}{@{}@{}l@{}}$D$: $\mathbb{E}_{p_e}[\log(1+e^{-D(x,z)})]+\mathbb{E}_{p_g}[\log(1+e^{D(x,z)})]$ \\ $E$: $\mathbb{E}_{p_e}[D(x,z)]$ \\ $G$: $-\mathbb{E}_{p_g}[e^{D(x,z)}]$\end{tabular}} \\\midrule
$f$-GAN & \hspace{-2cm}\small{\begin{tabular}{@{}@{}l@{}}$D$: $-\mathbb{E}_{p_e}[D(x,z)]+\mathbb{E}_{p_g}[e^{D(x,z)-1}]$ \\ $E$: $\mathbb{E}_{p_e}[D(x,z)]$ \\ $G$: $-\mathbb{E}_{p_g}[e^{D(x,z)-1}]$ \end{tabular}} \\\bottomrule
\end{tabular}
\label{tab:kl_obj}
\end{table}

We summarize the major differences between $f$-GAN and AGES for $f$-divergence minimization as follows:
\begin{itemize}\vspace{-0.05in}
\setlength{\itemsep}{1pt}
\setlength{\parskip}{0pt}
\item Based on Theorem \ref{thm:grad}, our framework provides a more unified treatment of various divergences than $f$-GAN: $f$-GAN uses different $D$ losses with artificially specified output activations $a_f$ for each divergence, while we obtain similar algorithms for various divergences with the only difference being the scaling in gradients.
\item For estimating the density ratio (or its function), AGES always applies Logistic regression which owns the highest asymptotic statistical efficiency, while $f$-GAN, limited by the minimax formulation, uses other losses except for JS divergence and hence is not as efficient.
\item For practical considerations, our scaling clipping technique addresses the unstable gradient issue via clear justification on the globally bounded scaling, while $f$-GAN is heuristically motivated following GAN.
\end{itemize}\vspace{-0.1in}


\section{Unifying VAEs and GANs}\label{sec:unify}
In this section we establish a unified view of VAEs and GANs. We regard BiGAN as the full version of GAN and point out that both VAEs and GANs are special cases of the general bidirectional formulation optimized using AGE, with different divergences and distribution assumptions.

\subsection{Variational Autoencoders}


VAEs \cite{kingma2013auto} learn the encoder $p_e(z|x)$ and the generator $p_g(x|z)$ by minimizing the negative variational lower bound or evidence lower bound (ELBO)
\begin{align}\label{eq:obj_vae}
\begin{split}
L_{\rm VAE}=-\mathbb{E}_{x\sim p_r}[\mathbb{E}_{z\sim p_e(z|x)}[\log p_g(x|z)]-\dkl(p_e(z|x),p_z(z))].
\end{split}
\end{align}
According to (\ref{eq:kl_decom_cc}), we have the following relationship between $L_{\rm VAE}$ and our objective in (\ref{eq:obj_fdiv}) with $D_f$ being KL:
\begin{equation*}
L_{\rm VAE}=\dkl(p_e(x,z),p_g(x,z))-\mathbb{E}_{p_r(x)}[\log p_r(x)].
\end{equation*}
Because the second term on the right-hand side is free of any learnable parameters, minimizing $L_{\rm VAE}$ is equivalent to minimizing (\ref{eq:obj_fdiv}). In the original VAE, both the encoder and generator distributions are set as factorized Gaussian distributions, leading to an analytic form of $L_{\rm VAE}$ that can be easily optimized. Therefore VAE is a special case of our general formulation optimized with AGES when $\mathcal{L}$ is KL and gradients can be evaluated analytically. 


However, the Gaussian assumption in the original VAE may not be expressive enough \cite{kingma2016improved,huszar2017variational}, especially for complex and high-dimensional data. Adversarial Variational Bayes (AVB) \cite{mescheder2017adversarial} extends the Gaussian encoder in VAE to an implicit distributions. Then the KL term in the objective (\ref{eq:obj_vae}) no longer has an explicit form. AVB introduces a discriminator $\mathcal{D}'(x,z)=\log(p_e(z|x)/p_z(z))$ and compute the gradient of the KL term w.r.t. encoder parameter $\phi$ as follows:
\begin{align*}
&\nabla_\phi\mathbb{E}_{x\sim p_r}[\dkl(p_e(z|x),p_z(z))]
=\mathbb{E}_{x\sim p_r,\epsilon\sim p_\epsilon}\left[\nabla_z \mathcal{D}'(x,E_\phi(x,\epsilon))^\top \nabla_\phi E_\phi(x,\epsilon)\right],
\end{align*}
which can be derived according to Theorem \ref{thm:grad} by noting that $f''(r)r=1$ for KL. 
Notice the relationship between $\mathcal{D}'$ and $\mathcal{D}$ defined in Theorem \ref{thm:grad}:
$\mathcal{D}'(x,z)=\mathcal{D}(x,z)+\log p_g(x|z)-\log p_r(x),$
where the difference only depends on learnable parameters through $\log p_g(x|z)$ that has an analytic form since AVB uses a Gaussian generator. Therefore, AVB can also be regarded as a special case of our formulation involving partial gradient estimation. 

\subsection{Bidirectional Generative Adversarial Networks}

BiGAN \cite{Donahue2017AdversarialFL,Dumoulin2017AdversariallyLI} directly adopts the original GAN in bidirectional generative modeling. With an additional encoder, it formulates the problem as a minimax game:
\begin{align*}
\min_{G,E}\max_D\quad &-\mathbb{E}_{p_e(x,z)}[\log(1+e^{-D(x,z)})]-\mathbb{E}_{p_g(x,z)}[\log(1+e^{D(x,z)})]\\
=&-\mathbb{E}_{x\sim p_r,\epsilon\sim p_\epsilon}[\log(1+e^{-D(x,E_\phi(x,\epsilon))})]-\mathbb{E}_{z\sim p_z,\epsilon\sim p_\epsilon}[\log(1+e^{D(G_\theta(z,\epsilon),z)})],
\end{align*}
where the equality follows from the reparametrization trick.
In our formulation (\ref{eq:obj_fdiv}) when we choose $D_f$ as JS, applying the formula in Theorem \ref{thm:grad}, we obtain the gradients as follows:
\begin{align*}
\nabla_\theta L&=-\mathbb{E}_{z\sim p_z,\epsilon\sim p_\epsilon}[\nabla_\theta\log(1+e^{D(G_\theta(z,\epsilon),z)})]\\
\nabla_\phi L &= -\mathbb{E}_{x\sim p_r,\epsilon\sim p_\epsilon}[\nabla_\phi\log(1+e^{-D(x,E_\phi(x,\epsilon))})]
\end{align*}
where the dependence of $D$ on parameters $\theta$ and $\phi$ is ignored when taking the gradients. 
Comparing them with the above minimax problem, we know that when $D$ is fixed, both formulations share the same form of gradients. Hence they are equivalent and BiGAN is again a special case of our bidirectional formulation with AGES.


\section{Experiments}

We evaluate our method in three aspects. First we investigate the performance of the proposed algorithm in divergence optimization, to verify that AGES can indeed minimize the divergence effectively. Second we explore the influence of different divergences and bidirectional formulations on the issue of mode collapse. Lastly we apply the bidirectional generative models learned with AGES on real datasets and test the performance in both generation and representation, which further shows the effectiveness of our method. All the details of experimental setup are given in Appendix \ref{app:exp}.\footnote{The code is available at \url{https://github.com/xwshen51/AGES}.}

\subsection{Divergence optimization}\label{sec:exp_opt}

In order to make a fair comparison, we consider the scenario where the original VAE applies and use the same objective function and model settings for different methods. Specifically, we choose $\mathcal{L}$ as the KL divergence and set both encoder $p_e(z|x)$ and generator $p_g(x|z)$ as factorized Gaussians. As a result, problem (\ref{eq:obj_bigen}) is equivalent to minimizing $L_{\rm VAE}$ (\ref{eq:obj_vae}) which has an analytic form so that we can compute the exact objective values for comparison.

\smallskip
\noindent
\textbf{Datasets} \ \ 
To make the model assumptions suitable for data, we synthesize a toy dataset from a 2D mixture of Gaussians (MoG) with 9 components laid out on a grid. We assume imbalanced class probabilities with 4 minority classes and 5 majority classes, which makes it a decently hard task.

\smallskip
\noindent
\textbf{Methods for comparison} \ \ 
The first one is VAE where we analytically minimize $L_{\rm VAE}$ using SGD. The second one is the proposed AGES with $\mathcal{L}$ being the KL divergence, which we call AGES-KL. The third one is the bidirectional $f$-GAN with KL, abbreviated as $f$-GAN-KL. Note that the solution obtained from VAE is regarded as the ``ideal" solution, since it makes use of the analytic form of the objective while the other two use estimated gradients and minimax approximation respectively. 

\smallskip
\noindent
\textbf{Metrics} \ \ 
We use three metrics to evaluate the performance. The first one is the value of objective function $L_{\rm VAE}$ which directly indicates the optimization performance. The second is $L_{\text{CC}}$ in (\ref{eq:cc}) to measure the cycle consistency. The last one is the marginal negative log-likelihood $-\mathbb{E}_{p_r(x)}[\log p_g(x)]$ to validate the performance in unidirectional generation. We estimate the first two metrics with samples and the third one using the annealed importance sampling (AIS) \cite{wu2016quantitative} with 1000 intermediate distributions and 30 parallel chains on 10,000 test examples.

As reported in Table \ref{tab:ais}, AGES is comparable to VAE in all three metrics, indicating that our proposed method can minimize the KL objective almost as effective as optimizing the closed-form objective when available. Thus AGES is a good alternative of VAE especially when we use more general encoder/generator distributions to which VAE does not apply. In contrast, $f$-GAN performs far worse due to the low statistical efficiency of the discriminator. Moreover, $f$-GAN is highly unstable with a large variability between multiple repetitions, which is also observed in \cite{nowozin2016f}. This experiment directly suggests that AGES outperforms $f$-GAN in both effective optimization and training stability. 

\begin{table}[h]
\centering\small
\caption{Metrics for the objective value, cycle consistency and unidirectional generation (the smaller the better). All results are averaged over 10 trials shown with the standard error.}
\vskip 0.1in
\begin{tabular}{c|ccc}
\toprule
Method & Objective & CC & Uni-gen \\\midrule
VAE\hspace{-4pt} & 2.739 (0.02)\hspace{-5pt} & 0.025 (0.07)\hspace{-5pt} & 0.753 (0.01) \\
AGES-KL\hspace{-4pt} & 2.784 (0.06)\hspace{-5pt} & -0.018 (0.16)\hspace{-5pt} & 0.737 (0.02) \\
$f$-GAN-KL\hspace{-4pt} & 3.786 (1.05)\hspace{-5pt} & 1.173 (1.04)\hspace{-5pt} & 1.401 (1.24) \\
\bottomrule
\end{tabular}
\label{tab:ais}
\end{table}

\subsection{Mode coverage}

In this section we focus on the influence of the choice of divergence on the issue of mode coverage. Note that orthogonal to methods that target on solving mode collapse 
\cite{che2016mode,srivastava2017veegan,lin2018pacgan}, our discussion here only considers the factor of divergence. 

\smallskip
\noindent
\textbf{Datasets} \ \ 
We consider two scenarios. One is a synthetic MoG dataset like above while we add the number of components to 25 with 12 minority classes. In this case, the 25 modes have imbalanced probabilities and the minority modes could be easily lost. The other is the stacked MNIST dataset \cite{che2016mode,lin2018pacgan}, which is constructed by stacking three randomly sampled MNIST digits. Hence it has 1000 modes with uniform probabilities.

\smallskip
\noindent
\textbf{Methods for comparison} \ \ 
We mainly focus on different choices of divergence as the objective with AGES for optimization. We compare KL, JS, Reverse KL, and Squared Hellinger distance. In addition, we also compare with $f$-GAN-KL, our proposed AGES-ALL, and two external baselines non-saturating BiGAN \cite{Donahue2017AdversarialFL,Dumoulin2017AdversariallyLI} with the ``logD" trick (abbreviated as logD-GAN) and Hinge loss used in BigBiGAN \cite{donahue2019large} for comparison. 

\smallskip
\noindent
\textbf{Metrics} \ \ 
We use two previously used metrics. One is the number of modes captured by a generator. For the above two labelled datasets, we can compute this number using pre-trained classifiers. Another metric is the reverse KL divergence (since the KL divergence is infinity when some mode is missing) between the mode distribution of generated samples and the real mode distribution (which are discrete and tractable).

\subsubsection{MoG}

We use deterministic encoders and generators in this experiment. 
The results are reported in Table \ref{tab:mog}, from which we can see that KL divergence has a significant benefit in mode capturing over other divergences or formulations. Reverse KL performs far worse than KL, and the Squared Hellinger distance which is defined in between KL and reverse KL performs moderately. AGES-KL tends to be better and more stable than $f$-GAN-KL, which is consistent with the results in Section \ref{sec:exp_opt}. AGES-ALL is slightly worse than several divergences on this toy dataset, but still covers more modes than Reverse KL, logD-GAN and Hinge. 
Figure \ref{fig:mog} visually shows the reconstruction performance of various divergences with additional results given in Appendix \ref{app:samples}, where we can clearly observe how mode collapse occurs for all divergences except KL. 

\begin{table}
\centering\small
\caption{Two measures of mode collapse on the imbalanced MoG dataset. All results are averaged over 10 trials shown with the standard error.}
\vskip 0.1in
\begin{tabular}{c|cc}
\toprule
Method & Modes & KL\\\midrule
AGES-KL & \textbf{24.9} (0.36) & \textbf{0.0284} (0.0035) \\
$f$-GAN-KL & 24.1 (1.33) & 0.0477 (0.0418) \\
AGES-$H^2$ & 24.5 (0.94) & 0.0439 (0.0102) \\
AGES-RevKL & 21.8 (3.21) & 0.2498 (0.2120) \\
GAN (JS) & 24.1 (1.23) & 0.0462 (0.0159) \\
AGES-ALL & 23.2 (1.33) & 0.1133 (0.0273) \\
logD-GAN \cite{Donahue2017AdversarialFL,Dumoulin2017AdversariallyLI} & 20.0 (3.14) & 0.3437 (0.1723) \\
Hinge \cite{donahue2019large} & 20.9 (0.92) & 0.1929 (0.0366) \\
\bottomrule
\end{tabular}
\label{tab:mog}
\end{table}

\begin{figure}
\centering
\subfigure[Real]{
\includegraphics[width=0.22\textwidth]{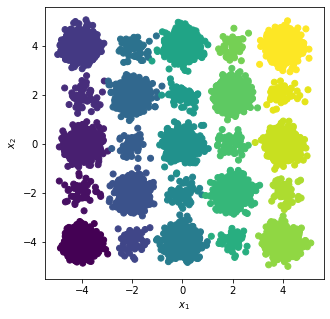}}
\subfigure[AGES-KL]{
\includegraphics[width=0.22\textwidth]{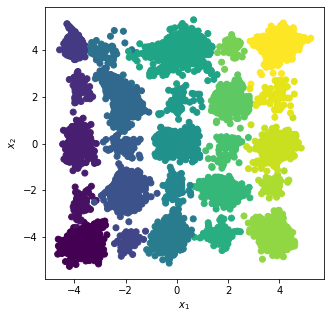}}
\subfigure[AGES-$H^2$]{
\includegraphics[width=0.22\textwidth]{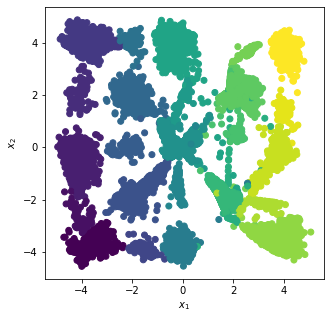}}
\subfigure[AGES-RevKL]{
\includegraphics[width=0.22\textwidth]{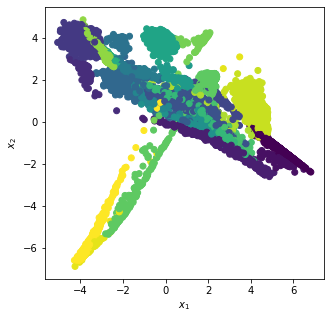}}\vspace{-0.1in}
\caption{Reconstructions from bidirectional generative models on the MoG dataset using various divergences as the objective.}
\label{fig:mog}
\end{figure}

\subsubsection{Stacked MNIST}
On real datasets, we find that scaling clipping is necessary in order to maintain stable gradients. In appendix \ref{app:sc} we show how the AGES algorithms behave with varying clipping ranges and conclude that AGES-ALL generally performs well and stably. Thus, for real data tasks we adopt AGES-ALL that approximately minimizes several $f$-divergences (including KL) simultaneously with stable training. We try a deterministic encoder and generator called ``AGES-ALL(d)" and use Gaussian encoders and implicit generators for all other methods. The details of adding randomness is given in Appendix \ref{app:choice_dist}.

The results in Table \ref{tab:mnist} demonstrate the effectiveness of AGES-ALL in diminishing mode collapse on Stacked MNIST. Furthermore, we observe advantages of stochastic encoders and generators over deterministic ones in both mode covering and reconstruction accuracy (discussed later in Section \ref{sec:real}). This is consistent with the arguments that stochasticity in transformations increases the expressiveness of generative models and adding noise to the generator helps alleviate mode collapse. We notice that $f$-GAN-KL tends to perform poorly and even collapse on this dataset, leading to a far worse result.

\begin{table}
\centering\small
\caption{Two measures of mode collapse and reconstruction accuracy on Stacked MNIST. All results are averaged over 10 trials shown with the standard error.}
\vskip 0.1in
\begin{tabular}{c|ccc}
\toprule
Method & Modes & KL & Recon.(\%)\\\midrule
AGES-ALL(d) & 971.7 (20.1) & 0.42 (0.10) & 81.7 (1.8) \\
AGES-ALL & \textbf{981.2} (9.5) & \textbf{0.36} (0.05) & \textbf{86.5} (1.6) \\
$f$-GAN-KL & 466 (452.5) & 3.49 (2.45) & 27.8 (1.6) \\
GAN (JS) & 954.2 (17.3) & 0.71 (0.08) & 64.2 (2.2) \\
logD-GAN \cite{Donahue2017AdversarialFL,Dumoulin2017AdversariallyLI} & 932.1 (59.8) & 0.55 (0.14) & 81.3 (2.3) \\
Hinge \cite{donahue2019large} & 959.9 (17.1) & 0.53 (0.10) & 84.1 (1.8) \\
\bottomrule
\end{tabular}
\label{tab:mnist}
\end{table}

\subsection{Real data generation and representation} \label{sec:real}

In this section we apply our method on real datasets of digits (Stacked MNIST), human faces (CelebA \cite{liu2015deep}) and natural images (ImageNet \cite{russakovsky2015imagenet}) to extensively evaluate the performance of our method in data generation and representation. Stacked MNIST is an elementary dataset; CelebA contains a large number of well-aligned face images with large variations of attributes; ImageNet contains real-world images with a huge diversity and thus is one of the most elusive tasks in image synthesis. 

For fair comparison, we mainly consider three approaches with non-saturating losses and high training stability on real datasets: AGES-ALL (proposed), and two previous state-of-the-art bidirectional generative models: Hinge (BigBiGAN \cite{donahue2019large}) and logD-GAN (BiGAN \cite{Donahue2017AdversarialFL,Dumoulin2017AdversariallyLI}). For all methods, we apply Gaussian encoders and implicit generator distributions with details given in Appendix \ref{app:choice_dist}.
Due to limited computational resource, we resize the images from CelebA and ImageNet to the resolution of $64\times64$ and use relatively small network architectures and training scale with details given in Appendix \ref{app:exp}. 

\subsubsection{Generation}

Generated samples on three datasets are shown in Figure \ref{fig:gen}, with the Fr\'echet Inception Distances (FIDs) \cite{heusel2017gans} reported in Table \ref{tab:fid}. Additional samples are presented in Appendix \ref{app:samples}. The results demonstrate the advantage of our method to generate images with high fidelity, which is a consequence of effective optimization and merits of our bidirectional generative formulation. 

Furthermore, we find that the bidirectional generative models (BGMs) achieve comparable performance to unidirectional generative models (UGMs, row 1 in Table \ref{tab:fid}). One explanation is our justification on the advantage of the bidirectional formulation in unidirectional generation. For ImageNet with such a huge diversity, the generator in a BGM benefits from the encoder and achieves much better performance than that in a UGM. 
Hence, bidirectional generative models should be favored over unidirectional ones since they can achieve the goal of the latter while additionally learn an inference model which is useful in many applications.
\begin{figure*}
\centering
\subfigure[Stacked MNIST]{
\includegraphics[width=0.2\textwidth]{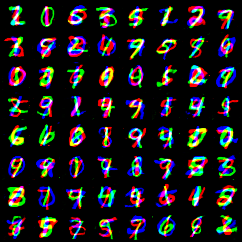}}
\subfigure[CelebA]{
\includegraphics[width=0.3\textwidth]{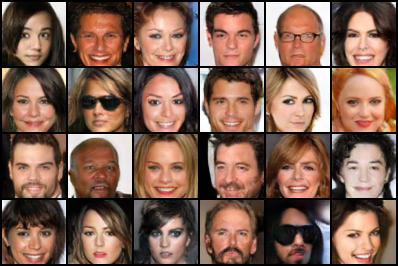}}
\subfigure[ImageNet]{
\includegraphics[width=0.45\textwidth]{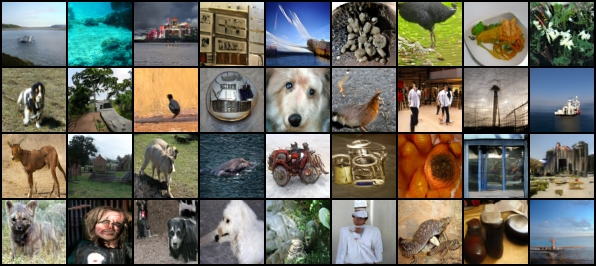}}\vspace{-0.15in}
\caption{Generations from AGES-ALL on three datasets.}
\label{fig:gen}
\vskip -0.15in
\end{figure*}

\begin{table}
\centering\small
\caption{FIDs on three datasets. Our AGES-ALL outperforms other bidirectional approaches on three datasets under same experimental settings, while is comparable to large scale BigBiGAN on ImageNet which uses much larger networks, batch sizes and higher resolution of input images.}
\vskip 0.1in
\begin{tabular}{c|ccc}
\toprule
Method & Stacked MNIST & CelebA & ImageNet \\\midrule
AGES-UGM & 4.89 & 8.91 & 19.33 \\\hline
AGES-ALL & \textbf{4.40} & \textbf{8.51} & \textbf{16.38} \\
Hinge \cite{donahue2019large} & 5.63 & 10.04 & 19.02 \\
BigBiGAN \cite{donahue2019large} & - & - & \textbf{15.82}\footnotemark \\
logD-GAN \cite{Donahue2017AdversarialFL,Dumoulin2017AdversariallyLI} & 5.56 & 11.98 & 19.81 \\
\bottomrule
\end{tabular}
\label{tab:fid}
\vskip -0.1in
\end{table}
\footnotetext{This is the result on $64\times64$ ImageNet reported in \cite{donahue2019large}.}

\subsubsection{Representation}


In order to explore the property of the latent representations learned by our BGM, we investigate the reconstruction performance, latent space interpolation, and nearest neighbors. 


We would like to investigate how much information, especially high-level semantics, is preserved in the inferred representation $E(x)$ by looking at the reconstruction $G(E(x))$. Since our concern is not in the pixel level, we measure the reconstruction performance by how much high-level features or attributes it can retain. We use both qualitative illustration and quantitative metrics. 
The last column of Table \ref{tab:mnist} reports the classification accuracy of the reconstructions on Stacked MNIST and shows the advantage of AGES in preserving category information.
Figure \ref{fig:celeba_recon} and \ref{fig:imagenet_recon} present the reconstructions on CelebA and ImageNet validation sets, with additional samples given in Appendix \ref{app:samples}. AGES achieves much more faithful reconstructions than other methods, which supports our theoretical justifications on cycle consistency in Section \ref{sec:cc}. Although the reconstructions are generally not perfect in the pixel level, our method is able to capture high-level attributes and semantics. 
This property is essentially demanded in learning causal representations and is worth investigating in future work. 

\begin{figure}
\centering
\includegraphics[width=0.5\textwidth]{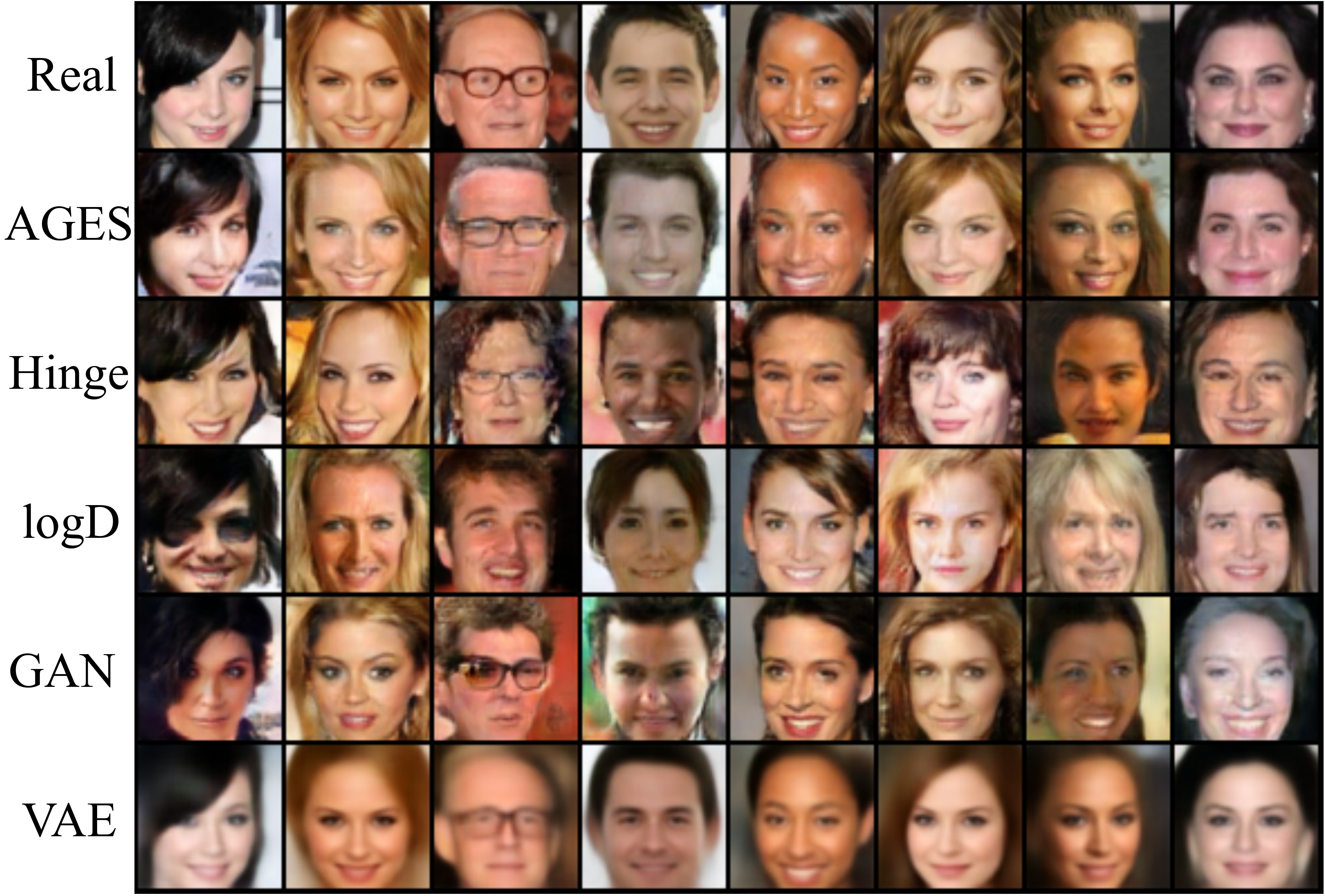}\vspace{-0.1in}
\caption{Reconstructions on CelebA. The reconstructions from AGES are sharp and tend to share the same attributes as the original images, such as, azimuth, emotion, hair/skin color, glasses, etc.}
\label{fig:celeba_recon}
\vskip -0.1in
\end{figure}
\begin{figure}
\centering
\includegraphics[width=0.5\textwidth]{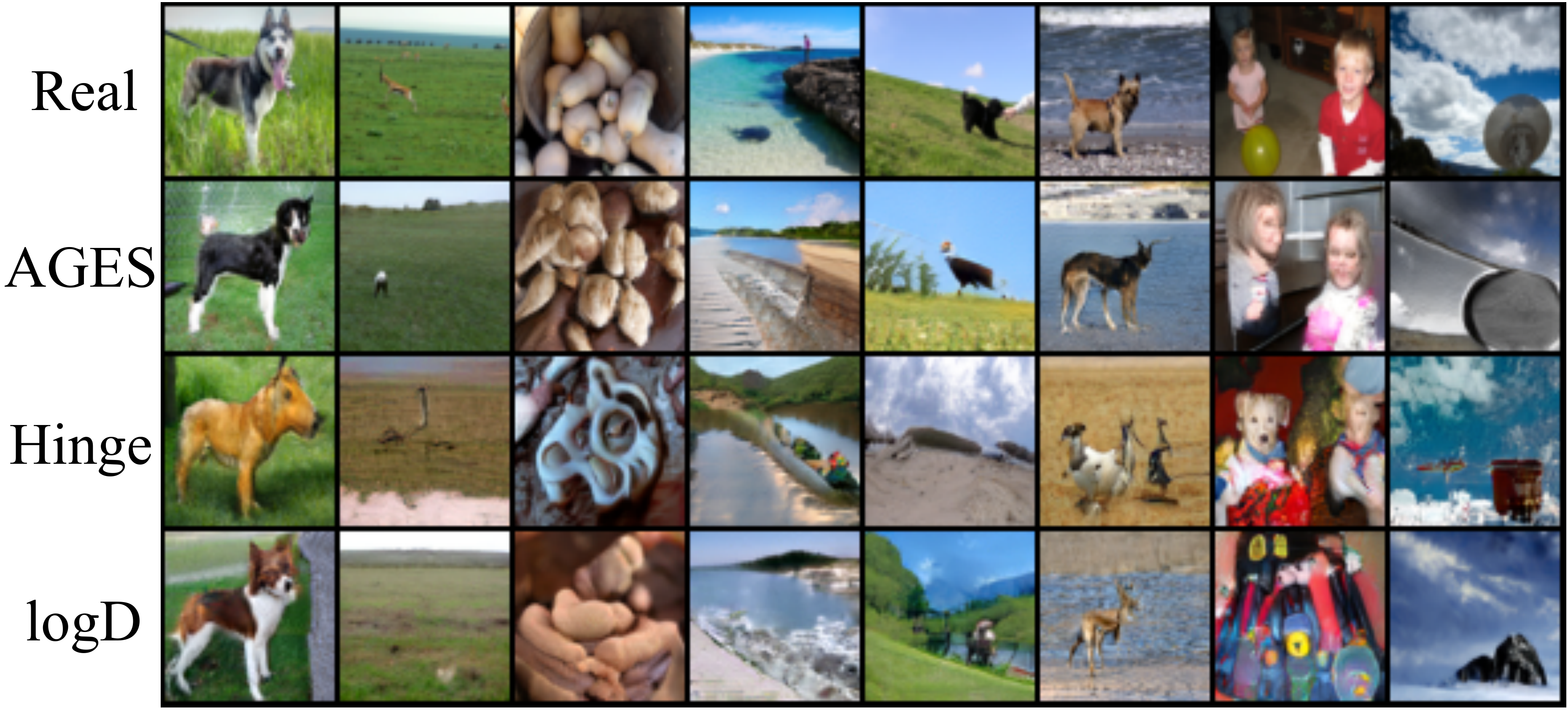}\vspace{-0.1in}
\caption{Reconstructions on ImageNet. The reconstructions from AGES are more often belonging to the same category as the original images with similar texture, position, and pose.}
\label{fig:imagenet_recon}
\end{figure}

\begin{figure}
\centering
\subfigure[Latent interpolation]{\label{fig:interp}
\includegraphics[width=0.665\textwidth]{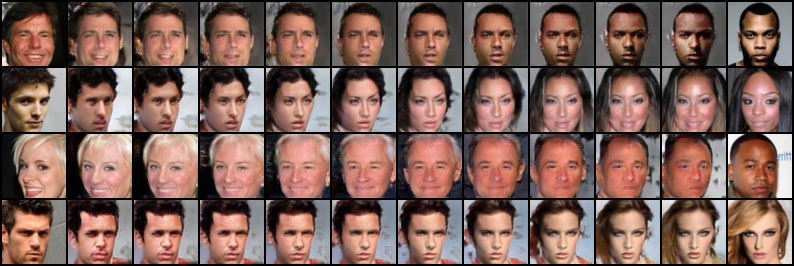}}
\hspace{.1in}
\subfigure[Nearest neighbors]{\label{fig:nn}
\includegraphics[width=0.28\textwidth]{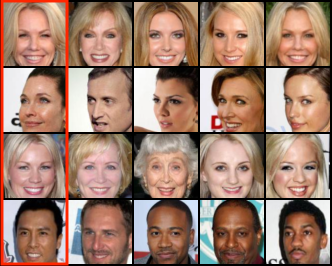}}
\vspace{-0.1in}
\caption{(a) Latent space interpolations on CelebA validation set using AGES. The left and right columns are real images; the columns in between are generated from the latent variables interpolated linearly from the two inferred representations of the real. (b) Nearest neighbors in the learned latent space. The one in the red rectangle is a query image, and the remaining ones are its four nearest neighbors. All images are from the validation set.}
\end{figure}

Figure \ref{fig:interp} shows latent space interpolations between validation samples which exhibit smooth semantic transitions, verifying the smooth and well-dispersed latent space learned by our model. 
As shown in Figure \ref{fig:nn}, the neighbors in the latent space often share the same high-level features with the query image, indicating that the learned representations are mostly consistent with visual semantics. Results from other methods in Appendix \ref{app:samples} suggest the advantage of ours.


\section{Related work}

\subsection{Bidirectional generative modeling}

VAE is often regarded to be far different from GANs. We point out in Section \ref{sec:unify} that both are special cases of our bidirectional generative modeling optimized using AGES. Both are limited to specific objectives and model assumptions, and hence do not apply to the general $f$-divergence formulation with more expressive generator/encoder distributions. 

Along the extensions, apart from AVB \cite{mescheder2017adversarial} which is also a special case of ours involving partial gradient estimation, approaches like VAE-GAN \cite{larsen2015autoencoding} and AAE \cite{makhzani2015adversarial} or more general WAE \cite{Tolstikhin2017WassersteinA} enhance VAE using unidirectional GANs. The former uses a GAN to match the data distributions while the latter uses one in the latent space. These formulations are mainly motivated by certain specific concerns and design the loss accordingly, and thus are not as principled as ours. Other methods including ALICE \cite{li2017alice} and VEEGAN \cite{srivastava2017veegan} can be regarded as variants of our fundamental formulation by adding special regularizers like conditional entropy or reconstruction error on the latent space. 
Recently, BigBiGAN \cite{donahue2019large} is proposed to implement the formulation of BiGAN using the BigGAN architectures. Its main contribution is to translate the progress in image generation to representation learning, especially the network architectures with much more capacity and benefits of scaling up training. In contrast, our work only considers small training scale while focus on the formulation and algorithm, and thus is orthogonal to it.

\subsection{Adversarial approach for $f$-divergence minimization}\label{sec:relate_div_min}

There is a number of work involving adversarial approach for $f$-divergence minimization. One principled approach is the $f$-GAN \cite{nowozin2016f} which is based on the variational representation of $f$-divergences. We investigate clearly the differences and advantages of our AGES over $f$-GAN in Section \ref{sec:fgan} and experiments.

Several papers decompose a problem of $f$-divergence minimization into two subproblems: density ratio estimation and divergence minimization \cite{poole2016improved,mohamed2016learning,mescheder2017adversarial,srivastava2017veegan,chen2017symmetric}, but their methods are fundamentally different from ours. They directly estimate the objective by the discriminator and the consequent algorithms are heuristic based on adversarial training. Specifically when evaluating the gradients they simply ignore the dependence in the discriminator on parameters while only take into account the dependence in data. Their derived gradient estimations are generally different from ours given the same objective. Besides, natural ways to extend these methods to bidirectional cases tend to diverge. In Appendix \ref{app:related}, we give more detailed discussion and comparison in both theoretical forms and empirical performance.

CFG-GAN \cite{Johnson2019AFO} presents a new framework for GANs using functional gradient learning where the generator is updated by adding an estimator of the functional composition. In contrast, we follow the traditional GANs with parametrized networks and Theorem \ref{thm:grad} enables us to directly evaluate the gradient wrt. the parameters.

\subsection{Gradient estimation in generative modeling}

There are some literature involving score estimation in the context of generative modeling where the score $\nabla_x \log q(x)$ of a given probability density $q(x)$ is the gradient of interest. One method of score estimation is the Stein gradient estimator \cite{li2017gradient,shi2018spectral} proposed for implicit distributions. This is further applied to estimate the gradient of mutual information which is a special case of KL \cite{wen2020mutual}. \cite{song2019generative} presents a new generative model where samples are produced via Langevin dynamics using gradients of the data distribution estimated with score matching. However this paper considers the gradient of a general $f$-divergence wrt. the encoder/decoder parameter which cannot be formulated regarding the score function and thus cannot be estimated based on the previous work.

\subsection{Unifying VAEs and GANs}

Some literature propose to unify VAEs and GANs. \cite{hu2017unifying} reformulates GANs and VAEs under the framework of Adversarial Domain Adaptation and links them back to the classic wake-sleep algorithm. To achieve this, the authors sometimes regard latent variables as observed ones and generation process as inference, which may not be as straightforward. In contrast, this paper starts with a general formulation of bidirectional generative modeling followed by a proposed optimization approach, which turns out to accommodate both VAE and GAN under our framework. Hence our unified view is more natural and directly related to generative models. AVB \cite{mescheder2017adversarial} proposes a specific approach to combine VAE and GAN, rather than a unified view in a broad sense as ours, as discussed in Section \ref{sec:unify}.


\section{Conclusion}

This paper considers the general $f$-divergence formulation of bidirectional generative modeling and discuss its advantages. 
We propose a new optimization method, AGES, for this formulation, where the gradient is computed using an adversarially learned discriminator. In our framework, we show that different divergences induce similar algorithms in terms of gradient evaluation, except with different scaling. This unifies the treatment of $f$-divergence GAN. Therefore this paper proposes a general recipe for a class of principled $f$-divergence based generative modeling methods. We further propose the scaling clipping technique and obtain an algorithm which locally minimizes several divergences simultaneously with a lower variance and higher training stability.  

Extensive empirical studies are conducted to demonstrate the advantages of our approach over existing methods, including effective divergence optimization, alleviating mode collapse, and promising performance in real data generation and representation. The potential of our method in more applications such as disentanglement/causal representation learning, image translation and downstream classification tasks, and the benefits after scaling up are worth further exploration.

{
\nocite{*} 
\bibliography{ref_nips.bib}
\bibliographystyle{ieeetr}}

\clearpage


\begin{appendices}


\section{Unidirectional $f$-divergence minimization with AGES} \label{app:uni_age}

In this section we apply AGES in unidirectional generation, where our goal is to learn the stochastic transformation $G_\theta(z,\epsilon)$ so that its distribution, denoted by $p_\theta(x)$, becomes close to the real data distribution $p_r(x)$. Similar to the formulation in bidirectional case, we consider the following optimization problem
\begin{equation*}
\min_{\theta}\ L_{\rm uni}(\theta)=D_f(p_r(x),p_\theta(x))
\end{equation*}
which coincides with the problem stated in $f$-GAN. However equipped with AGES, we generally obtain different algorithms given the same $f$-divergence as the objective. The following theorem enables us to evaluate the gradient of $f$-divergence w.r.t. the generator parameter.
\begin{thm}
Let $r(x)=p_r(x)/p_\theta(x)$ and $\mathcal{D}(x)=\log(p_r(x)/p_\theta(x))$. We have
\begin{equation}\label{eq:grad_uni}
\nabla_\theta L_{\rm uni}(\theta)=-\mathbb{E}_{z\sim p_z,\epsilon\sim p_\epsilon}\left[s(G_\theta(z,\epsilon)) \nabla_x\mathcal{D}(G_\theta(z,\epsilon))^\top\nabla_\theta G_\theta(z,\epsilon)\right],
\end{equation}
where $s(x)=\tilde{f}''\left(1/r(x)\right)/r(x)$.
\end{thm}
\begin{proof}
Similar to the proof of Theorem \ref{thm:grad}.
\end{proof}

Let $D(x)$ be the solution to the empirical Logistic regression that distinguishes the generated data from the real data:
\begin{equation*}
D(x)=\argmin_{D'}\left[\frac{1}{|S_r|}\sum_{x\in S_r}\log(1+e^{-D'(x)}) + \frac{1}{|S_g|}\sum_{x\in S_g}\log(1+e^{D'(x)})\right]
\end{equation*}
where $S_r$ and $S_g$ are finite samples from $p_r(x)$ and $p_g(x)$ respectively. Similarly we know $D(x)\approx\mathcal D(x)$.
Replacing $\mathcal D(x)$ and $r(x)$ in the gradients (\ref{eq:grad_uni}) with $D(x)$ and $\hat{r}(x)=e^{D(x)}$, we obtain the maximum likelihood estimator for the gradients. Similar to bidirectional AGES-ALL, we obtain an algorithm that approximately minimizes several $f$-divergences between $p_r(x)$ and $p_g(x)$ simultaneously by setting $s(x)=1$.


\section{Detailed discussion on related work}
\label{app:related}
A number of papers propose to use discriminator-based approaches for minimizing KL or more general divergences. In this section we give a detailed discussion on the fundamental difference of those methods from ours. 

As mentioned in many papers \cite{poole2016improved,mohamed2016learning,uehara2016generative}, a problem of $f$-divergence minimization can be decomposed into two subproblems: density ratio estimation and divergence minimization. Our proposed method is coherent in this sense. 

For density ratio estimation, existing methods use Logistic regression \cite{mescheder2017adversarial,srivastava2017veegan,chen2017symmetric} or $f$-GAN $D$ losses \cite{poole2016improved}. As mentioned in main text, Logistic regression, which is also used in our method, is motivated by its asymptotic statistical efficiency, while $f$-GAN $D$ losses are not as efficient and have different forms for various divergences, which is not as unified and easy to implement as Logistic.

More crucially, for divergence minimization, all previous methods are fundamentally different from ours. In a word, previous methods estimate the objective and the consequent algorithms are heuristic based on the idea of adversarial training. In contrast, the derived gradient formula in Theorem \ref{thm:grad} enables us to directly estimate the gradient. Hence our algorithm which is based on gradient descent has guarantee for convergence following the convergence results of SGD and consistency of density ratio estimation.

To be specific, as pointed out by \cite{poole2016improved}, $f$-divergences are a family of divergences that depend only on samples from one distribution and the density ratio. Based on this, they first estimate the objective function by plug in the density ratio estimator. Recall the objective in (\ref{eq:obj_fdiv})
\begin{equation}\label{eq:objapp}
L(\theta,\phi)=\bE_{p_\theta(x,z)}[f(r(x,z))]=\bE_{p_\phi(x,z)}[\tilde{f}(1/r(x,z))],
\end{equation}
where $p_\theta(x,z)=p_g(x,z)$, $p_\phi(x,z)=p_e(x,z)$ and $r(x,z)=p_\phi(x,z)/p_\theta(x,z)$. They obtain the estimated objective $\hat{L}(\theta,\phi)=\bE_{p_\theta(x,z)}[f(e^{D(x,z)})]$, where $D(x,z)$ gives an estimate for $\log(p_\phi(x,z)/p_\theta(x,z))$, e.g., the solution to the empirical Logistic regression (\ref{eq:obj_logistic}).


Their consequent algorithms are based on alternating optimization. When evaluating the gradient of the estimated objective wrt. the generator(/encoder) parameters, they only take the samples part into account while ignore the dependence of the density ratio estimator itself on the parameters. Specifically, they update $\theta$ using gradient 
\begin{equation*}
\frac{1}{n}\sum_{i=1}^n \nabla_\theta f(e^{D(G_\theta(z_i,\epsilon_i),z_i)})=\frac{1}{n}\sum_{i=1}^n f'(e^{D(G_\theta(z_i,\epsilon_i),z_i)})e^{D(G_\theta(z_i,\epsilon_i),z_i)}\nabla_x D(G_\theta(z_i,\epsilon_i),z_i)^\top \nabla_\theta G_\theta(z_i,\epsilon_i)
\end{equation*}
which is generally not the true gradient derived in Theorem \ref{thm:grad}, as shown more concretely in Table \ref{tab:diff_scale}. Besides, this method cannot be directly applied in bidirectional case since the samples only depend on one of encoder and decoder while we need to learn both. A possible way one may think of is to use the equivalent expression of $f$-divergence as in (\ref{eq:objapp}) and then update $\phi$ similarly using gradient
\begin{equation*}
\frac{1}{n}\sum_{i=1}^n \nabla_\phi \tilde{f}(e^{-D(x_i,E_\phi(x_i,\epsilon_i))})=-\frac{1}{n}\sum_{i=1}^n \tilde{f}'(e^{-D(x_i,E_\phi(x_i,\epsilon_i))})e^{-D(x_i,E_\phi(x_i,\epsilon_i))}\nabla_z D(x_i,E_\phi(x_i,\epsilon_i))^\top \nabla_\phi E_\phi(x_i,\epsilon_i)
\end{equation*}
which again is generally not the true gradient derived in Theorem \ref{thm:grad}.

Now we uniformly denote the gradients used in SGD by
\begin{align*}
&\frac{1}{n}\sum_{i=1}^n \tilde{s}_\theta(G_\theta(z_i,\epsilon_i),z_i) \nabla_x D(G_\theta(z_i,\epsilon_i),z_i)^\top \nabla_\theta G_\theta(z_i,\epsilon_i)\\
&\frac{1}{n}\sum_{i=1}^n \tilde{s}_\phi(x_i,E_\phi(x_i,\epsilon_i)) \nabla_z D(x_i,E_\phi(x_i,\epsilon_i))^\top \nabla_\phi E_\phi(x_i,\epsilon_i)
\end{align*}
where $\tilde{s}_\theta(x,z)$ and $\tilde{s}_\phi(x,z)$ are scalings depending on the divergence and specific methods. Note that $\tilde{s}_\theta=-s_\theta$ and $\tilde{s}_\phi=s_\phi$ in Theorem \ref{thm:grad}. Table \ref{tab:diff_scale} lists the scalings of KL, RevKL and JS divergence derived from different methods, which shows that the gradients derived from previous methods generally differs from ours and even with contrary signs. This indicates that previous heuristic algorithms will not work in bidirectional $f$-divergence minimization.

We then verify the failure of their algorithm in bidirectional KL divergence minimization through the experiment in Section \ref{sec:exp_opt}. Figure \ref{fig:obj_value} shows the training curve of VAE (ideal one), AGES-KL and their algorithm.

\begin{table}[h]
\centering\small
\caption{Scalings in gradients of KL, RevKL and JS divergence derived from previous methods and our AGES}
\vskip 0.1in
\begin{tabular}{c|cc|cc|cc}
\toprule
Divergence & \multicolumn{2}{c|}{KL} & \multicolumn{2}{c|}{RevKL} & \multicolumn{2}{c}{JS} \\\midrule
Method & Others & AGES & Others & AGES & Others & AGES \\\midrule
$\tilde{s}_\theta$ & 1 & 1 & $(D-1)e^{-D}$ & $-e^{-D}$ & $-e^D\log\frac{2}{1+\exp(-D)}$ & $\frac{1}{1+\exp(-D)}$ \\
$\tilde{s}_\phi$ & $(D+1)e^D$ & $-e^D$ & $-1$ & $-1$ & $e^{-D}\log\frac{2}{1+\exp(D)}$ & $\frac{1}{1+\exp(D)}$ \\
\bottomrule
\end{tabular}
\label{tab:diff_scale}
\end{table}

\begin{figure}[h]
\centering
\includegraphics[width=0.65\textwidth]{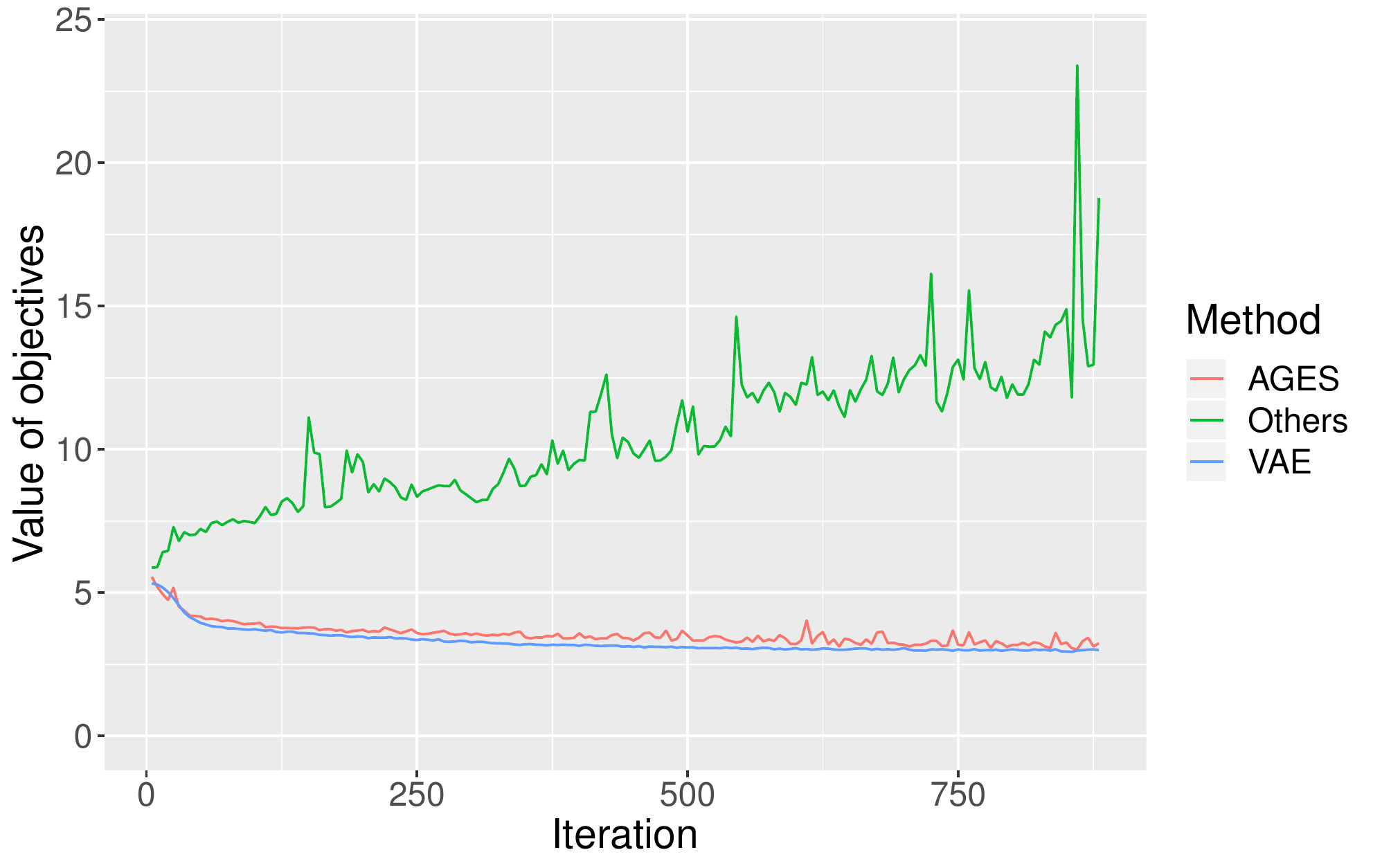}
\caption{Value of objective $L_{\text{VAE}}$ during training process of three methods.}
\label{fig:obj_value}
\end{figure}



\section{Choice for encoder and generator distributions} \label{app:choice_dist}

In this section we discuss the conventional choices for encoder and generator distributions in VAEs and GANs along with our suggestions.

\medskip
\noindent
\textbf{Gaussian generators may not be suitable.} \ 
While AVB extends the Gaussian encoder in VAE to an implicit distribution, we argue that the more crucial model element is the choice for the generator distribution. A Gaussian generator used in VAE and AVB is not suitable to model more complex real data like images. One intuitive explanation is that the complex space (e.g., pixels) is usually of very high dimension and may have some properties which can not be suited well in the Euclidean space with a Gaussian distribution. In practice when generating new data from VAE/AVB, after sampling $z\sim p_z$, people always use the mean of the Gaussian generator rather than random samples. For images, the means tend to be blurry while the random samples are noisy and far away from the true data distribution. The reason we suggest is due to the improper Gaussian assumption rather than KL as the objective. 

\medskip
\noindent
\textbf{Stochastic encoders and generators have benefits.} \ 
On the other hand, BiGAN uses deterministic encoder and generator transformations which are essentially degenerated cases in the sense that $p_e(z|x)$ and $p_g(x|z)$ can only capture one-point distributions. ALI and BigBiGAN use a deterministic generator and a stochastic encoder which causes some asymmetry. In contrast, we suggest that stochasticity in transformations increases model expressiveness and helps with mode covering.

\medskip
\noindent
\textbf{We suggest an implicit generator distribution.} \ 
Usually in unidirectional GANs, people use ``deterministic" generators but a relatively high dimensional latent vector. For example the progressive GAN \cite{karras2017progressive} sets the latent dimensionality to 512. We can think of the latent vector as a composition of latent factors to represent high-level features and random noises to capture stochastic variation. Then with the nonlinear transformation over the random noises, the conditional distribution of generated data given the high-level features is implicit. However in bidirectional models, the desired latent representation should only include high-level features. Thus we separate random noises $\epsilon$ as the source of randomness and follow the idea of expressive implicit distributions.
With an implicit generator distribution, both VAE and AVB do not apply and hence we are motivated to employ the proposed AGES for optimization.

\medskip
\noindent
\textbf{Details of the implicit generator.} \ 
To construct a generator with an implicit distribution, we adopt the similar idea as the StyleGAN generator \cite{karras2019style}. We generate some single-channel feature maps consisting of uncorrelated Gaussian noises, one for each layer of the generator network except the final output image, with the same resolution as the output feature map of that layer. Each noise feature map is broadcasted to all channels using learned per-pixel scaling factors and then added to the output of the corresponding convolution. In this manner, each convolution layer in the generator produces a conditional Gaussian distribution given all the previous layers. After the consequent nonlinear transformations, the final output image conditional on the input latent variable is an implicit distribution.

\section{Proofs}
\subsection{Proof of equivalent expressions of the joint KL}\label{app:kl_decom}

\begin{proof}[Proof of (\ref{eq:kl_decom_uni})]
By the definition of KL divergence, we have
\begin{align*}
\dkl(p_e(x,z),p_g(x,z))&=\mathbb{E}_{p_e(x,z)}\left[\log\frac{p_e(x,z)}{p_g(x,z)}\right]=\mathbb{E}_{p_e(x,z)}\left[\log\frac{p_r(x)p_e(z|x)}{p_g(x)p_g(z|x)}\right]\\
&=\mathbb{E}_{p_r(x)}\left[\log\frac{p_r(x)}{p_g(x)}\right] + \mathbb{E}_{p_r(x)}\left[\mathbb{E}_{p_e(z|x)}\left(\log\frac{p_e(z|x)}{p_g(z|x)}\right)\right]\\
&=\dkl(p_r(x),p_g(x))+\mathbb{E}_{x\sim p_r(x)}[\dkl(p_e(z|x),p_g(z|x))].
\end{align*}
\end{proof}

\begin{proof}[Proof of (\ref{eq:kl_decom_cc})]
\begin{align*}
\dkl(p_e(x,z),p_g(x,z))&=\mathbb{E}_{p_e(x,z)}\left[\log\frac{p_e(z|x)}{p_g(x|z)p_z(z)}\right]+\mathbb{E}_{p_e(x,z)}[\log p_r(x)]\\
&=-\mathbb{E}_{p_e(x,z)}[\log p_g(x|z)]+\mathbb{E}_{p_e(x,z)}\left[\log\frac{p_e(z|x)}{p_z(z)}\right]+\mathbb{E}_{p_r(x)}[\log p_r(x)]\\
&=-\mathbb{E}_{p_r(x)}\mathbb{E}_{p_e(z|x)}[\log p_g(x|z)]+\mathbb{E}_{p_r(x)}[\dkl(p_e(z|x),p_z(z))]+\mathbb{E}_{p_r(x)}[\log p_r(x)].
\end{align*}
\end{proof}

\subsection{Proof of Theorem \ref{thm:grad}}\label{app:thm_proof}
The proof technique is inspired by that of CFG-GAN \cite{Johnson2019AFO}.
Let $\|\cdot\|$ denote the vector 2-norm. Given a differentiable vector function $g(x):\mathbb{R}^k\to\mathbb{R}^k$, we use $\nabla\cdot g(x)$ to denote its divergence, defined as 
$$\nabla\cdot g(x):=\sum_{j=1}^k\frac{\partial[g(x)]_j}{\partial[x]_j},$$ 
where $[x]_j$ denotes the $j$-th component of $x$. We know that 
\begin{eqnarray*}
\int\nabla\cdot g(x)dx=0
\end{eqnarray*}
for all vector function $g(x)$ such that $g(\infty) = 0$. Given a matrix function $w(x)=(w_1(x),\dots,w_l(x)):\mathbb{R}^k\to\mathbb{R}^{k\times l}$ where each $w_i(x),i=1\dots,l$ is a $k$-dimensional differentiable vector function, its divergence is defined as $\nabla\cdot w(x)=(\nabla\cdot w_1(x),\dots,\nabla\cdot w_l(x))$.

To prove Theorem \ref{thm:grad}, we need the following lemma which specifies the dynamics of the generator joint distribution $p_g(x,z)$ and the encoder joint distribution $p_e(x,z)$, denoted by $p_\theta(x,z)$ and $p_\phi(x,z)$ here.

\begin{lemma}\label{lemma}
Using the definitions and notations in Theorem \ref{thm:grad}, we have 
\begin{align}
\nabla_\theta p_\theta(x,z) &= -\nabla_x p_\theta(x,z)^\top g_\theta(x) - p_\theta(x,z)\nabla\cdot g_\theta(x),\label{eq:grad_pgen}\\
\nabla_\phi p_\phi(x,z) &= -\nabla_z p_\phi(x,z)^\top e_\phi(z) - p_\phi(x,z)\nabla\cdot e_\phi(z),\label{eq:grad_penc}
\end{align}
for all data $x$ and latent variable $z$, where $g_\theta(G_\theta(z,\epsilon))=\nabla_\theta G_\theta(z,\epsilon)$ and $e_\phi(E_\phi(x,\epsilon))=\nabla_\phi E_\phi(x,\epsilon)$.
\end{lemma}

\begin{proof}[Proof of Lemma \ref{lemma}]
Let $l$ be the dimension of parameter $\theta$.
To simplify notation, let random vector $X=G_\theta(Z,\epsilon)\in\mathbb{R}^d$ and $Y=(X,Z)\in\mathbb{R}^{d+k}$, and let $p$ be the probability density of $Y$. For each $i=1,\dots,l$, let $\Delta=\delta e_i$ where $e_i$ is a $l$-dimensional unit vector whose $i$-th component is one and all the others are zero, and $\delta$ is a small scalar. Let $X'=G_{\theta+\Delta}(Z,\epsilon)$ and $Y'=(X',Z)$ so that $Y'$ is a random variable transformed from $Y$ by $$Y'=Y+ \begin{pmatrix} g(X) \\ \mathbf{0} \end{pmatrix} \Delta + o(\delta)$$ where $g(X)\in\mathbb{R}^{d\times l}$ and let $p'$ be the probability density of $Y'$. For an arbitrary $y'=(x',z)\in\mathbb{R}^{d+k}$, let $x=x'+g(x)\Delta+o(\delta)$ and $y=(x,z)$. Then we have
\begin{align}
p'(y')&=p(y)|\det(dy'/dy)|^{-1}\nonumber\\
&=p(y)|\det(I_d+\nabla g(x)\Delta+o(\delta))|^{-1}\nonumber\\
&=p(y)(1+\Delta^\top\nabla\cdot g(x)+o(\delta))^{-1}\nonumber\\ 
&=p(y)(1-\Delta^\top\nabla\cdot g(x)+o(\delta))\label{eq:l1}\\
&=p(y)-\Delta^\top p(y')\nabla\cdot g(x') +o(\delta)\label{eq:l2}\\
&=p(y')-\Delta^\top g(x')^\top\cdot\nabla_{x'} p(x',z) - \Delta^\top p(y')\nabla\cdot g(y') +o(\delta).\label{eq:l3}
\end{align}
The first three equalities use the multivariate change of variables formula for probability densities for the change from $Y$ to $Y'$ and the definition of determinant with terms explicitly expanded up to $O(\delta)$. (\ref{eq:l1}) uses the Taylor expansion of $(1+\xi)^{-1}=1-\xi+o(\xi)$ with $\xi=\Delta^\top\nabla\cdot g(y)$. (\ref{eq:l2}) follows from the fact that $p(y')=p(y)+o(1)$ and $\nabla\cdot g(x')=\nabla\cdot g(x)+o(1)$. (\ref{eq:l3}) is due to $p(y)=p(y')-(y'-y)^\top\cdot\nabla p(y')+o(\delta)$. 
Since $y'$ is arbitrary, above implies that
\begin{eqnarray*}
p'(x,z)=p(x,z)-\Delta^\top g(x)^\top\cdot\nabla_x p(x,z)-\Delta^\top p(x,z)\nabla\cdot g(x) + o(\|\delta\|)
\end{eqnarray*}
for all $x\in\mathbb{R}^d,z\in\mathbb{R}^k$ and $i=1,\dots,l$, leading to (\ref{eq:grad_pgen}) by taking $\delta\to0$, setting $g(x)=g_\theta(x)$, and noting that $p=p_\theta$ as both are the density of $(G_\theta(Z,\epsilon),z)$ and $p'=p_{\theta+\Delta}$ as both are the density of $(G_{\theta+\Delta}(Z,\epsilon),z)$. Similarly we can obtain (\ref{eq:grad_penc}).
\end{proof}

\begin{proof}[Proof of Theorem \ref{thm:grad}]
Rewrite the objective (\ref{eq:obj_fdiv}) as $L(\theta,\phi)=\int\ell(p_e,p_g)dxdz$ where $\ell(p_e,p_g)$ denotes the integrands in definition (\ref{eq:f_div}).
Let $\ell'_2(p_e,p_\theta)=\partial\ell(p_e,p_\theta)/\partial p_\theta$.
Using the chain rule and Lemma \ref{lemma}, we have 
\begin{align}
\nabla_\theta\ell(p_e(x,z),p_\theta(x,z))&=\ell'_2(p_e(x,z),p_\theta(x,z))\nabla_\theta p_\theta(x,z)\nonumber\\
&= \ell'_2(p_e(x,z),p_\theta(x,z))\left[-\nabla_x p_\theta(x,z)^\top g_\theta(x) - p_\theta(x,z)\nabla\cdot g_\theta(x)\right]\nonumber\\
&= p_\theta(x,z)\nabla_x \ell'_2(p_e(x,z),p_\theta(x,z))^\top g_\theta(x) - \nabla_x\cdot\left[\ell'_2(p_e(x,z),p_\theta(x,z))p_\theta(x,z)g_\theta(x)\right],\label{eq:grad_int}
\end{align}
where the third equality is obtained by applying the product rule as follows
\begin{align*}
\nabla_x\cdot\left[\ell'_2(p_e(x,z),p_\theta(x,z))p_\theta(x,z)g_\theta(x)\right]&=\ell'_2(p_e(x,z),p_\theta(x,z))p_\theta(x,z)\nabla\cdot g_\theta(x) \\
&\ + \ell'_2(p_e(x,z),p_\theta(x,z))\nabla_x p_\theta(x,z)^\top g_\theta(x) \\
&\ + p_\theta(x,z)\nabla_x \ell'_2(p_e(x,z),p_\theta(x,z))^\top g_\theta(x).
\end{align*}
By integrating (\ref{eq:grad_int}) over $x$ and $z$, and by using the fact that $\int\nabla\cdot f(x,z)dxdz=\mathbf{0}$ with \\$f(x,z)=\ell'_2(p_e(x,z),p_\theta(x,z))p_\theta(x,z)g_\theta(x)$, we have
\begin{eqnarray*}
\nabla_\theta L(\theta,\phi) = \int\nabla_\theta\ell(p_e(x,z),p_\theta(x,z))dxdz=\int p_\theta(x,z)\nabla_x \ell'_2(p_e(x,z),p_\theta(x,z))^\top g_\theta(x)dxdz.
\end{eqnarray*}
According to the definition (\ref{eq:f_div}) of $f$-divergences, we have  
\begin{eqnarray}\label{eq:lll}
\nabla_x \ell'_2(p_e(x,z),p_\theta(x,z))=\tilde{f}''\left(\frac{1}{r(x,z)}\right)\nabla_x \frac{1}{r(x,z)}=\tilde{f}''\left(\frac{1}{r(x,z)}\right)\frac{1}{r(x,z)}\nabla_x \mathcal D(x,z).
\end{eqnarray}
Further by reparametrization and noting that $r(x,z)=e^{D(x,z)}$, we obtain
\begin{align*}
\nabla_\theta L(\theta,\phi) &= -\mathbb{E}_{(x,z)\sim p_g(x,z)}\left[\tilde{f}''\left(\frac{1}{r(x,z)}\right)\frac{1}{r(x,z)}\nabla_x \mathcal D(x,z)^\top g_\theta(x)\right]\\
&=-\mathbb{E}_{z\sim p_z(z),\epsilon\sim p_\epsilon}\left[\tilde{f}''\left(\frac{1}{r(G_\theta(z,\epsilon),z)}\right)\frac{1}{r(G_\theta(z,\epsilon),z)}\nabla_x \mathcal D(G_\theta(z,\epsilon),z)^\top \nabla_\theta G_\theta(z,\epsilon)\right].\\
\end{align*}
Similarly we obtain 
\begin{align*}
\nabla_\phi L(\theta,\phi)&=\mathbb{E}_{(x,z)\sim p_e(x,z)}\left[f''(r(x,z))r(x,z)\nabla_z \mathcal D(x,z)^\top e_\phi(z)\right]\\
&= \mathbb{E}_{x\sim p_r(x),\epsilon\sim p_\epsilon}\left[f''(r(x,E_\phi(x,\epsilon)))r(x,E_\phi(x,\epsilon))\nabla_z \mathcal D(x,E_\phi(x,\epsilon))^\top \nabla_\phi E_\phi(x,\epsilon)\right].
\end{align*}
\end{proof}

\section{Additional experiments on scaling clipping}\label{app:sc}

In this section, we explore how the AGES algorithms behave with varying clipping ranges.
Following Section \ref{sec:sc}, we clip the scaling factors into a range of $[r_0,1/r_0]$ with a specified lower bound $r_0\in[0,1]$. The case of $r_0=0$ means no clipping while $r_0=1$ is the extreme case where all AGES algorithms for different divergences converge to AGES-ALL. 

Figure \ref{fig:sc} exhibits how various metrics vary with the lower bound of scaling clipping $r_0$ increasing from $0$ to $1$ on Stacked MNIST. Experiments show that without SC, AGEs for all divergences tend to suffer from the vanishing or exploding gradient problem and perform poorly. Reverse KL is the most stable one in this case.
We notice that some divergences perform better in certain metrics with certain clipping ranges while some perform better in other cases. For example, AGES-KL-SC has advantages in mode covering over other divergences, which coincides with the results on the MoG dataset. With heavy enough scaling clipping, different divergences do not differ too much on this dataset in reconstruction and generation. 
As we narrow the clipping range (increase $r_0$), the behaviors of different divergences converge to the same one. The extreme AGES-ALL performs stably and sufficiently well in all metrics, for which reason we adopt AGES-ALL on real datasets in the main text. 

We observe that $r_0\approx0.5$ is probably a decent choice where the AGES algorithms for different divergences significantly differ from each other and can preserve the distinctive property of each divergence, while avoid vanishing or exploding gradient. We hence report the detailed results of AGES-SC with a clipping range of $[0.5,2]$ on Stacked MNIST and CelebA in Table \ref{tab:sc}. We see that equipped with the scaling clipping technique, AGES algorithms for various divergences generally perform well and stably on real datasets. 

\begin{figure}[h]
\centering
\subfigure[Number of modes covered in generations (larger is better)]{
\includegraphics[width=0.48\textwidth]{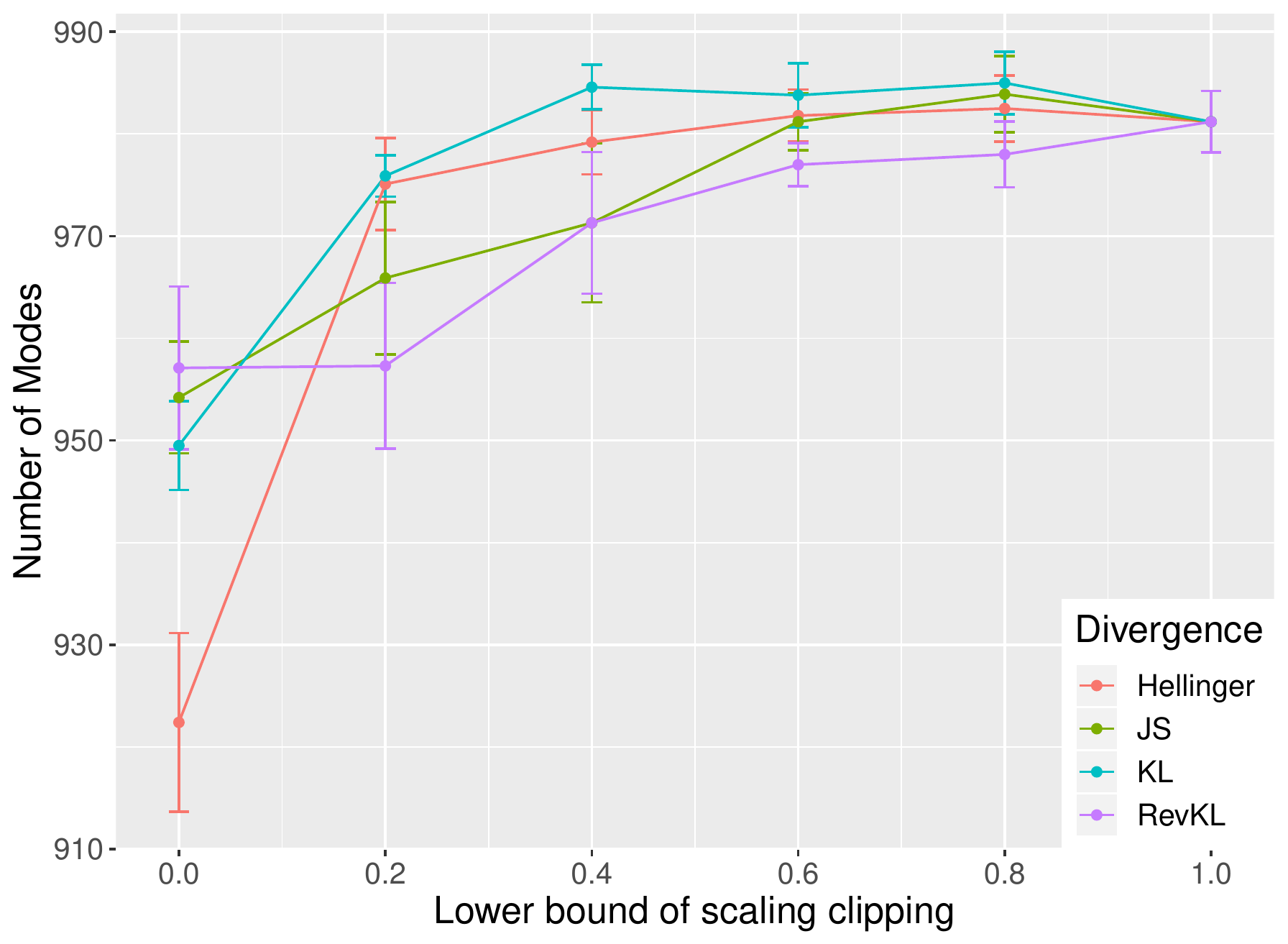}}
\subfigure[RevKL of real/generated mode dist. (smaller is better)]{
\includegraphics[width=0.48\textwidth]{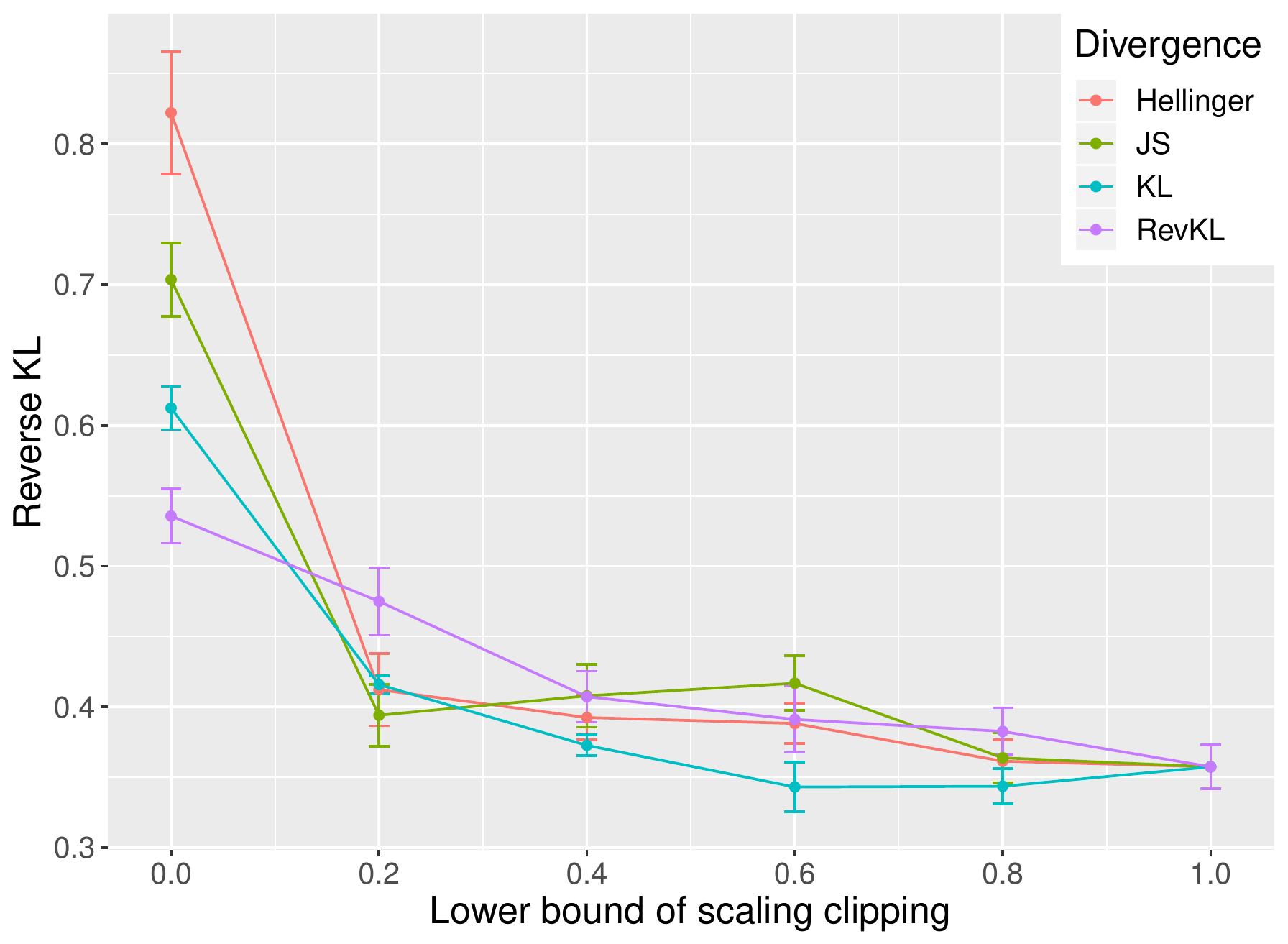}}
\subfigure[Reconstruction accuracy (higher is better)]{
\includegraphics[width=0.48\textwidth]{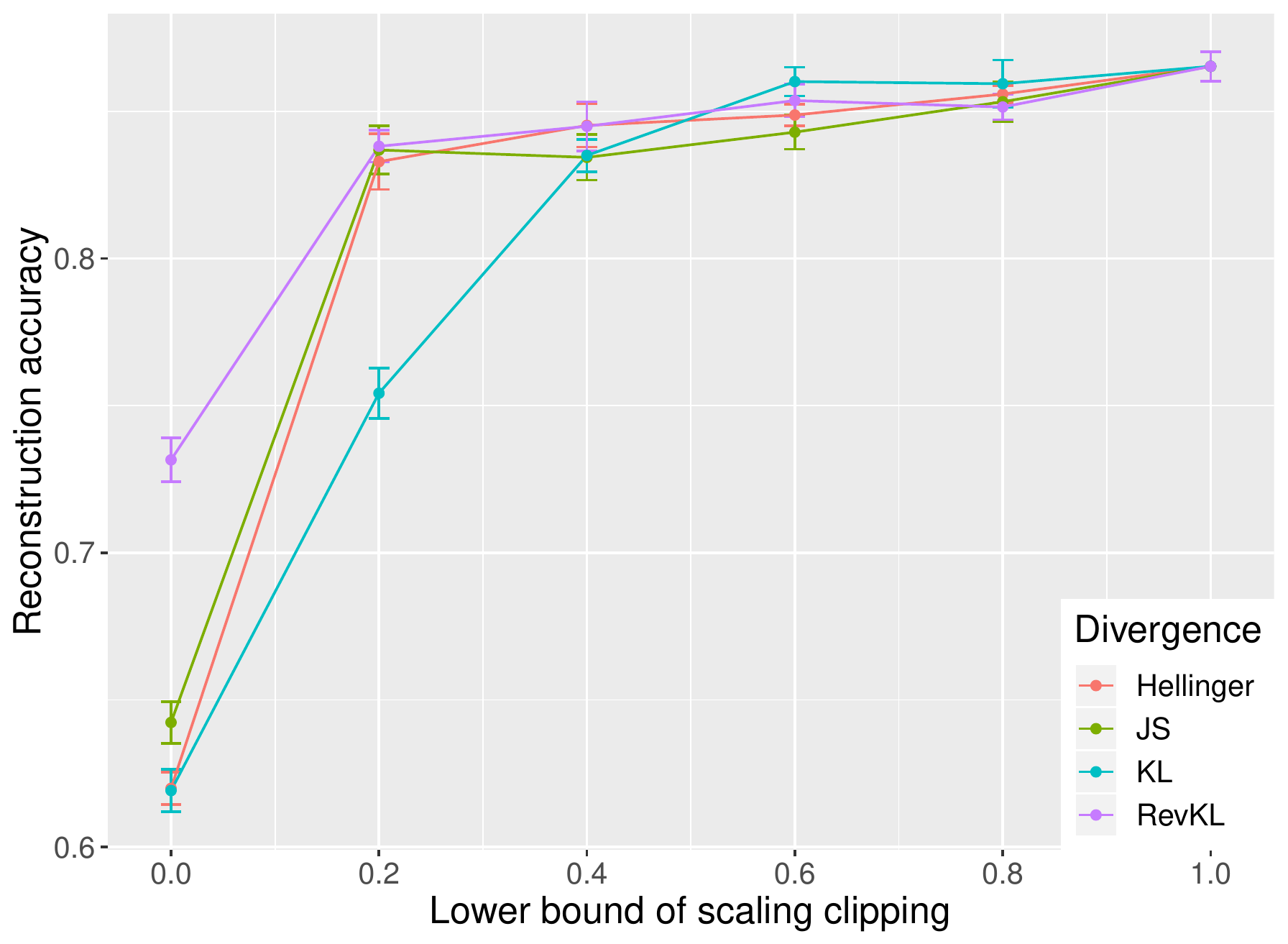}}
\subfigure[FID (smaller is better)]{
\includegraphics[width=0.48\textwidth]{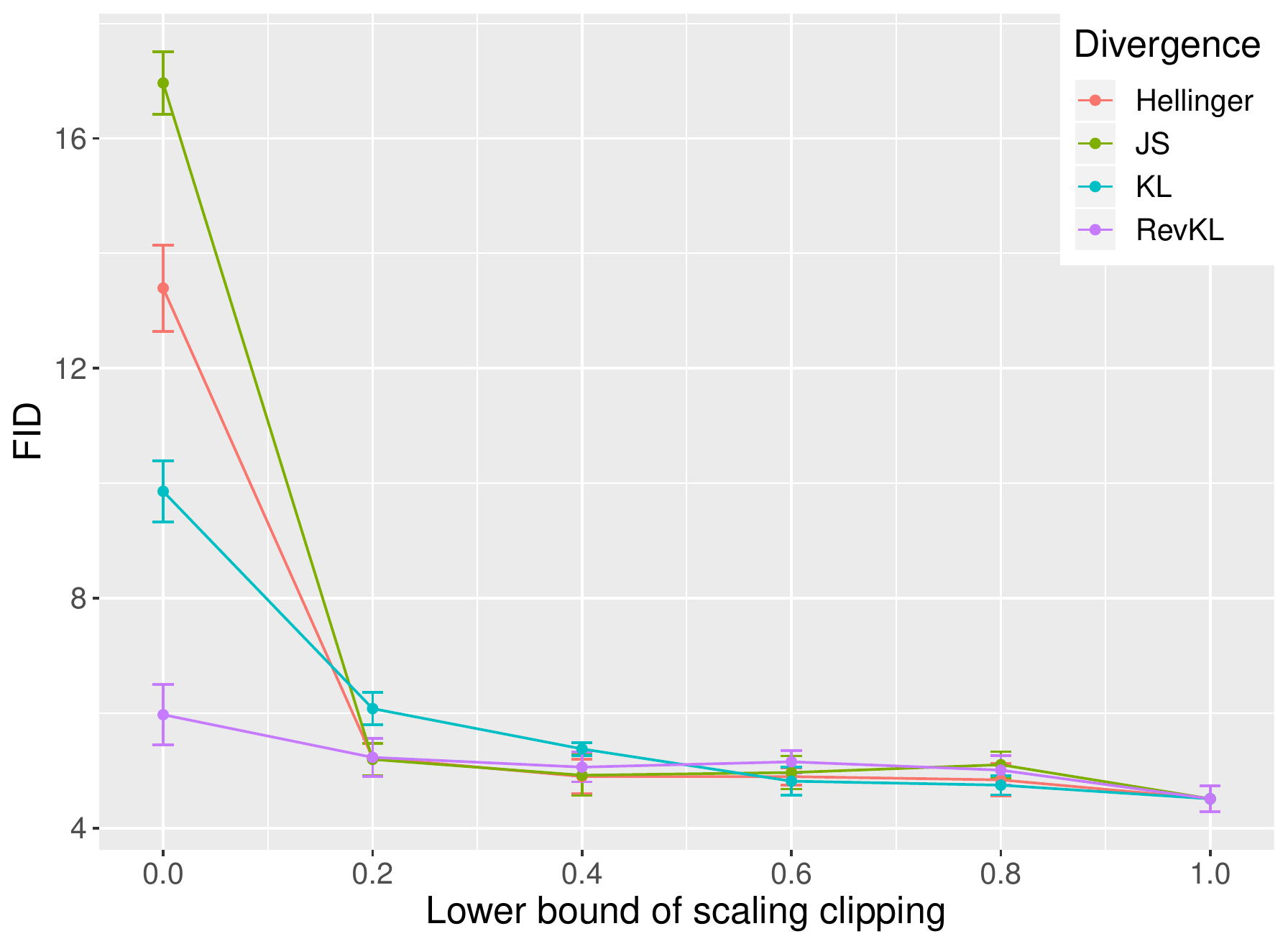}}
\caption{Behavior of AGES algorithms with scaling clipping (AGES-SC) on Stacked MNIST in mode covering, data reconstruction and generation. We repeat each experiment 10 times with the standard error shown by the error bar.}
\label{fig:sc}
\end{figure}

\begin{table}[h]
\centering
\caption{Performance of AGES algorithms with scaling clipping (AGES-SC) on real datasets. }
\subtable[Mode coverage and reconstruction accuracy on Stacked MNIST]{
\begin{tabular}{c|ccc}
\toprule
Divergence & Modes & KL & Recon.(\%)\\\midrule
ALL & 981.2 (9.5) & 0.3574 (0.05) & 86.53 (1.58) \\
KL & 983.5 (8.7) & 0.3503 (0.04) & 84.84 (1.51) \\
Hellinger & 982.5 (10.2) & 0.3497 (0.08) & 85.58 (0.92) \\
JS & 977.7 (17.7) & 0.3600 (0.05) & 83.70 (2.25) \\
RevKL & 976.9 (12.1) & 0.3690 (0.07) & 84.73 (1.26) \\
\bottomrule
\end{tabular}
}
\subtable[FIDs on real datasets]{
\begin{tabular}{ccc}
\toprule
Stacked MNIST & CelebA \\\midrule
4.40 & 8.51 \\
4.68 & 8.83 \\
5.04 & 8.73 \\
5.36 & 9.76 \\
5.06 & 9.93 \\
\bottomrule
\end{tabular}
}
\label{tab:sc}
\end{table}

\section{Experimental details}\label{app:exp}

In this section we state the details of experimental setup and the network architectures used for all experiments. 
In experiments on one dataset, we adopt exactly the same network architecture and experimental settings for different methods. On real datasets, we implement the previous SOTA methods BiGAN \cite{donahue2019large} and BigBiGAN \cite{Donahue2017AdversarialFL,Dumoulin2017AdversariallyLI} rather than directly using the results reported in the original papers for two reasons: (i) for fair comparison we implement them under exactly the same settings as our methods; (ii) the BiGAN paper used the old DCGAN architecture and did not report quantitative metrics for generation but just visually presented some generated samples; (iii) BigBiGAN used large training scale which we cannot afford and only consider one dataset ImageNet.

\subsection{MoG}

For both 9-Gaussians and 25-Gaussians datasets, each majority class contains 10,000 samples and each minority class contains 500 samples. The standard deviation is 0.3 for all classes. The generator and encoder have two and three fully connected layers respectively with 500 units in each layer with batch-normalization and ReLU as the activation function. The discriminator consists of three modules of two fully connected layers with 400 units each and Leaky-ReLU as the activation function to extract features from $x$, $z$ and their concatenated features. We use the Adam optimizer with a learning rate of $1\times10^{-4}$ for $D$ and $5\times10^{-5}$ for $E$ and $G$ and a mini-batch size of 500. The models on 25-Gaussians are trained for 30 epochs before evaluation. We use 30 $D$ steps per $G/E$ step on 9-Gaussians to retain a nearly optimal $D$, and 5 $D$ steps per $G/E$ step on 25-Gaussians to make it a harder task.

\subsection{Stacked MNIST}
We adopt the DCGAN \cite{radford2015unsupervised} architecture for Stacked MNIST.
When following exactly the same experimental setup reported in PacGAN \cite{lin2018pacgan} and VEEGAN \cite{srivastava2017veegan}, we find that all of the algorithms can cover all modes. Hence we reduce the model capacity to make it a harder task. Specifically, details for networks are given below in Table \ref{tab:mnist_g}-\ref{tab:mnist_d}. We use a pre-trained MNIST classifier to classify simulated samples on each of the three stacked channels. We train all models on 128,000 samples, with a mini-batch size of 64, for 50 epochs. We use Adam with a learning rate of 0.0001 and update all three networks once on each mini-batch. Evaluation for mode covering is done on 26,000 test samples. In all experiments, we use 50k generated images for evaluating FIDs. 

\begin{table}[h]
\centering\small
\caption{Generator network for Stacked-MNIST. With batch-normalization. With one Gaussian feature map added to each conv layer.}
\vskip 0.1in
\begin{tabular}{|c|c|c|c|c|}
\hline
Layer & Number of outputs & Kernel size & Stride & Activation function \\\hline
Input $z\sim\mathcal{N}(0,1)^{8}$ & 8 &-&-&-  \\
Fully-connected & $4\times4\times256$ &-&-& ReLU \\
Transposed convolution & $7\times7\times128$ & $5\times5$ & 2 & ReLU \\
Transposed convolution & $14\times14\times64$ & $5\times5$ & 2 & ReLU \\
Transposed convolution & $28\times28\times3$ & $5\times5$ & 2 & Tanh \\\hline
\end{tabular}
\label{tab:mnist_g}
\vskip -0.2in
\end{table}

\begin{table}[h]
\centering\small
\caption{Encoder network for Stacked-MNIST. With batch-normalization. The number of outputs is twice the latent dimension with a Gaussian encoder.}
\vskip 0.1in
\begin{tabular}{|c|c|c|c|c|}
\hline
Layer & Number of outputs & Kernel size & Stride & Activation function \\\hline
Input $x$ & $28\times28\times3$ &-&-&-  \\
Convolution & $14\times14\times64$ & $5\times5$ & 2 & ReLU \\
Convolution & $7\times7\times128$ & $5\times5$ & 2 & ReLU \\
Convolution & $4\times4\times256$ & $5\times5$ & 2 & ReLU \\
Fully-connected & 8 or 16 & - & - & - \\\hline
\end{tabular}
\label{tab:mnist_e}
\vskip -0.2in
\end{table}

\begin{table}[h]
\centering \small
\caption{Discriminator network for Stacked-MNIST. Without batch-normalization.}
\vskip 0.1in
\begin{tabular}{|c|c|c|c|c|}
\hline
Layer & Number of outputs & Kernel size & Stride & Activation function \\\hline
Input $x$ & $28\times28\times3$ &-&-&-  \\
Convolution & $14\times14\times64$ & $5\times5$ & 2 & LeakyReLU \\
Convolution & $7\times7\times128$ & $5\times5$ & 2 & LeakyReLU \\
Convolution & $4\times4\times256$ & $5\times5$ & 2 & LeakyReLU \\
Flatten &-&-&-&- \\
Concatenate $z$ &-&-&-&- \\
Fully-connected & 1024 & - & - & LeakyReLU \\
Fully-connected & 1 & - & - & - \\\hline
\end{tabular}
\label{tab:mnist_d}
\vskip -0.1in
\end{table}

\subsection{CelebA and ImageNet}
We pre-process the images by taking a center crops of $128\times128$ for CelebA and $73\times73$ for ImageNet and resizing to the $64\times64$ resolution.
For such complex datasets, we adopt the SAGAN \cite{zhang2018self,brock2018large} architecture for $D$ and $G$. For the discriminator, we adopt the similar idea in BigBiGAN, where we the $D$ network consists of three modules (Figure \ref{fig:arch_joint_d}) where $D_x$ is the normal SAGAN discriminator with data $x$ as input and feature $f_x$ and score $s_x$ as output, $D_z$ is an MLP with latent $z$ as input and score $s_z$ as output, and $D_{xz}$ is an MLP with concatenated feature $(f_x,f_z)$ as input and score $s_{xz}$ as output. Unlike BigBiGAN which introduces additional unary terms in the $D$ loss, we use a single output of $D$ as the average $(s_x+s_z+s_{xz})/3$ and keep the formulation of $D(x,z)$ -- Logistic regression between joint distributions $p_e(x,z)$ and $p_g(x,z)$. In this sense, involving unary scores here is just an architectural design for $D$ while in BigBiGAN makes it deviate from the original formulation (\ref{eq:obj_bigen}). Details for newtork $G$ and $D_x$ are given in Figure \ref{fig:arch_sagan} and Table \ref{tab:sagan}.
The encoder architecture is the ResNet50 \cite{he2016identity} followed by a 4-layer MLP (size 1024 for CelebA and 2048 for ImageNet) with skip connections after ResNet's global average pooling layer. 

We use Adam with $\beta_1=0$, $\beta_2=0.999$, and a learning rate of $1\times10^{-4}$ for $D$ and $5\times10^{-5}$ for $E$ and $G$. Due to limited computational resource, we use a mini-batch size of 256 for CelebA and 240 for ImageNet. We update all three networks once on each mini-batch. Models were trained for around 50 epochs on CelebA and 200 epochs on ImageNet on NVIDIA RTX 2080 Ti.

\begin{figure}[H]
\centering
\includegraphics[width=0.5\textwidth]{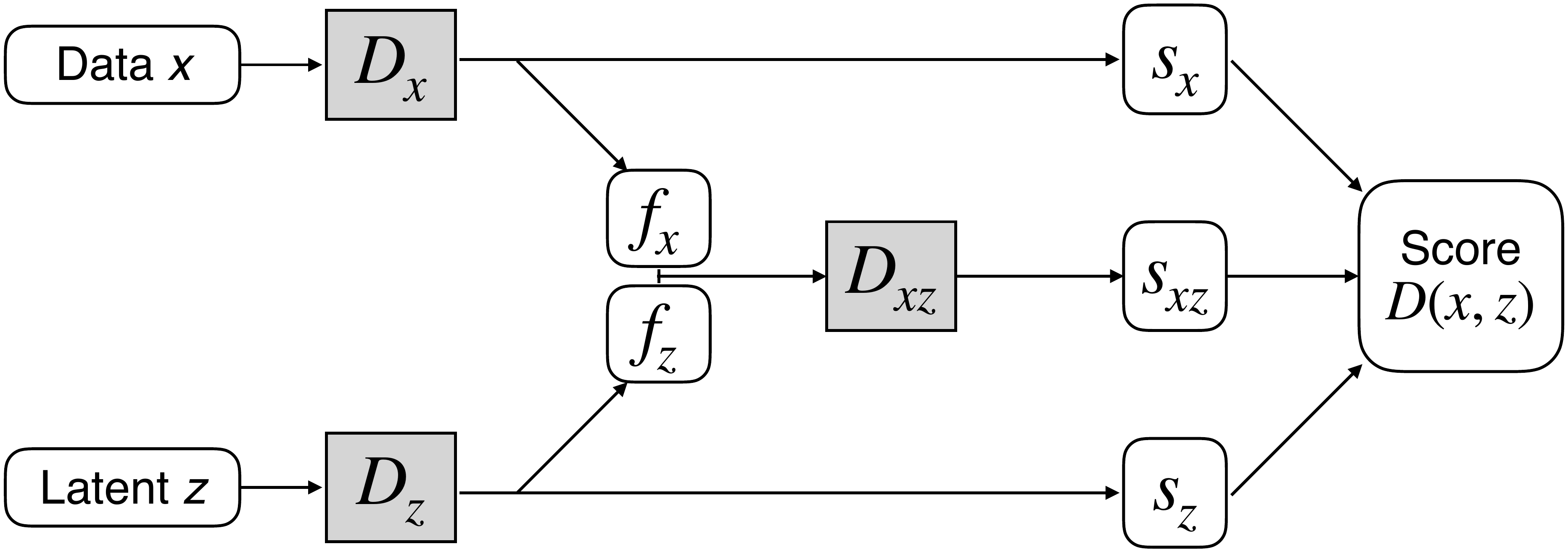}
\caption{Architecture of the discriminator $D(x,z)$}
\label{fig:arch_joint_d}
\end{figure}

\begin{figure}[H]
\centering
\subfigure[]{
\includegraphics[width=0.2\textwidth]{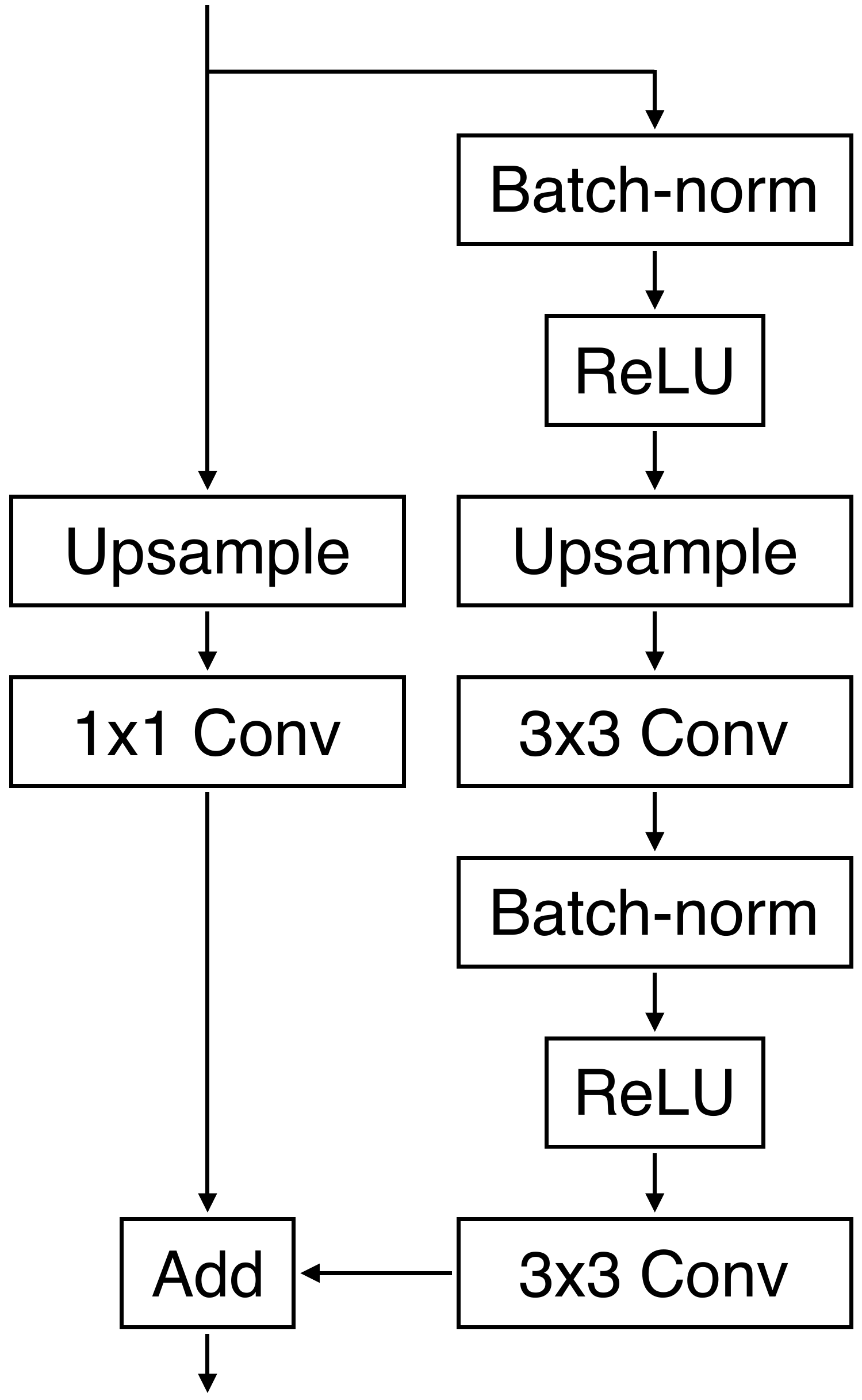}}\hspace{2cm}
\subfigure[]{
\includegraphics[width=0.2\textwidth]{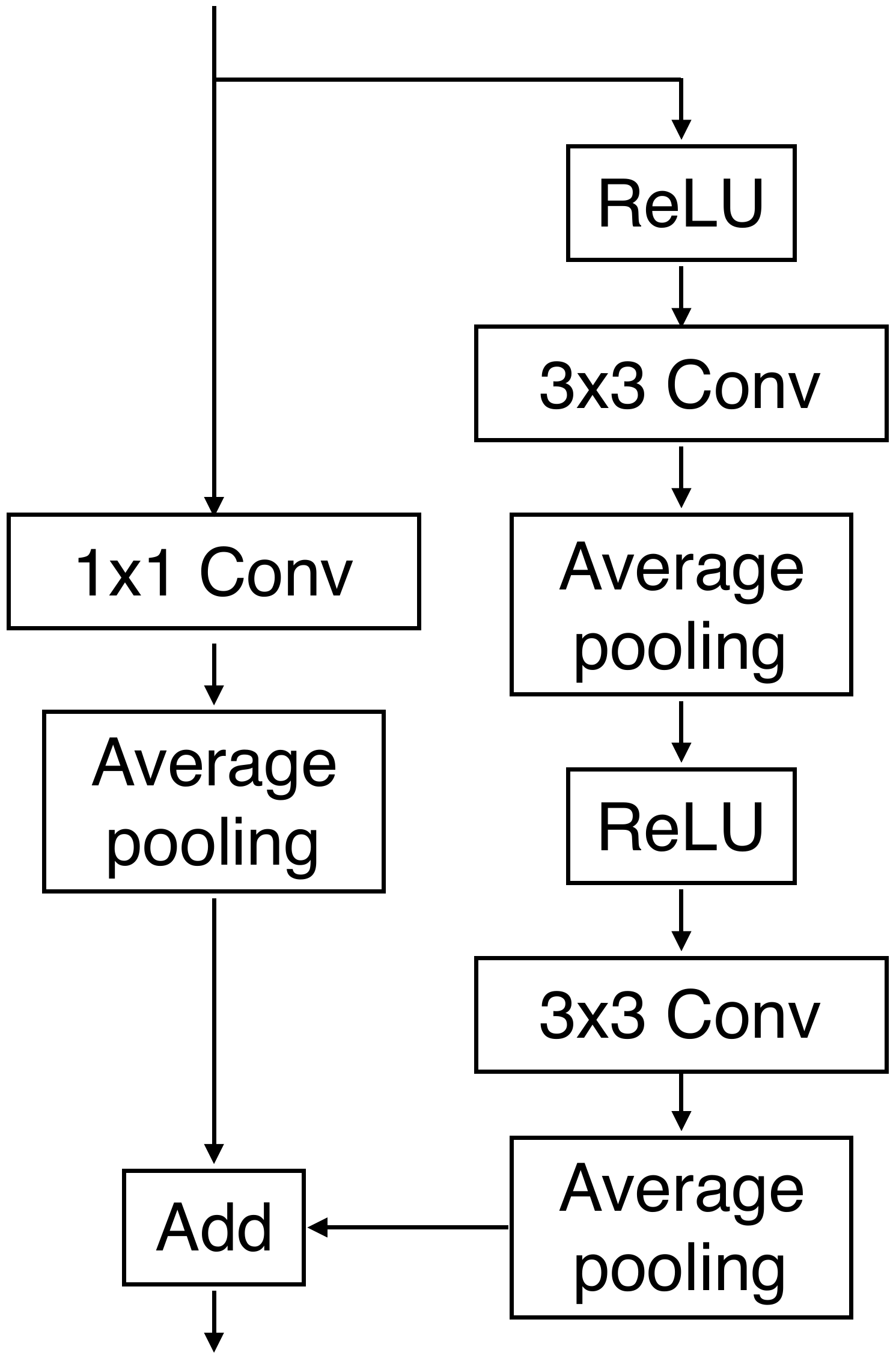}}
\vskip -0.1in
\caption{(a) A residual block (ResBlock up) in the SAGAN generator where we use nearest neighbor interpolation for upsampling; (b) A residual block (ResBlock down) in the SAGAN discriminator.}
\label{fig:arch_sagan}
\end{figure}

\begin{table}[H]
\centering\small
\caption{SAGAN architecture. CelebA uses $k=100$ and $ch=32$; ImageNet uses $k=140$ and $ch=64$.}
\subtable[Generator]{
\begin{tabular}{c}
\toprule
Input: $z\in\mathbb{R}^k\sim\mathcal{N}(0,I)$\\\midrule
Linear $\to4\times4\times16ch$\\\midrule
ResBlock up $16ch\to16ch$\\\midrule
ResBlock up $16ch\to8ch$\\\midrule
ResBlock up $8ch\to4ch$\\\midrule
Non-Local Block $(64\times64)$ \\\midrule
ResBlock up $4ch\to2ch$\\\midrule
BN, ReLU, $3\times3$ Conv $2ch\to3$\\\midrule
Tanh\\\bottomrule
\end{tabular}
}
\hspace{1cm}
\subtable[Discriminator module $D_x$]{
\begin{tabular}{c}
\toprule
Input: RGB image $x\in\mathbb{R}^{64\times64\times3}$\\\midrule
ResBlock down $ch\to2ch$\\\midrule
Non-Local Block $(64\times64)$ \\\midrule
ResBlock down $2ch\to4ch$\\\midrule
ResBlock down $4ch\to8ch$\\\midrule
ResBlock down $8ch\to16ch$\\\midrule
ResBlock $16ch\to16ch$\\\midrule
ReLU, Global average pooling ($f_x$)\\\midrule
Linear $\to1$ ($s_x$)\\\bottomrule
\end{tabular}
}
\label{tab:sagan}
\end{table}

\clearpage
\section{Additional samples and reconstructions}\label{app:samples}

\begin{figure}[h]
\centering
\subfigure{
\includegraphics[width=0.48\textwidth]{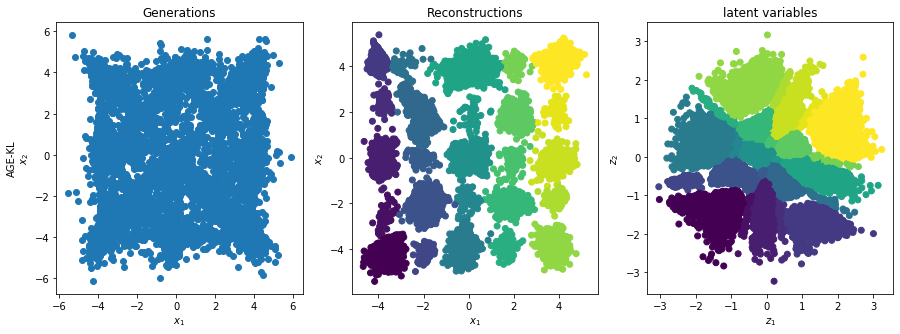}}
\subfigure{
\includegraphics[width=0.48\textwidth]{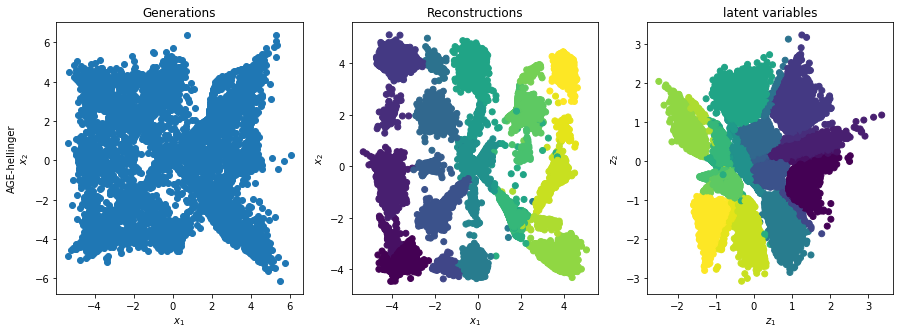}}
\vskip -0.17in
\subfigure{
\includegraphics[width=0.48\textwidth]{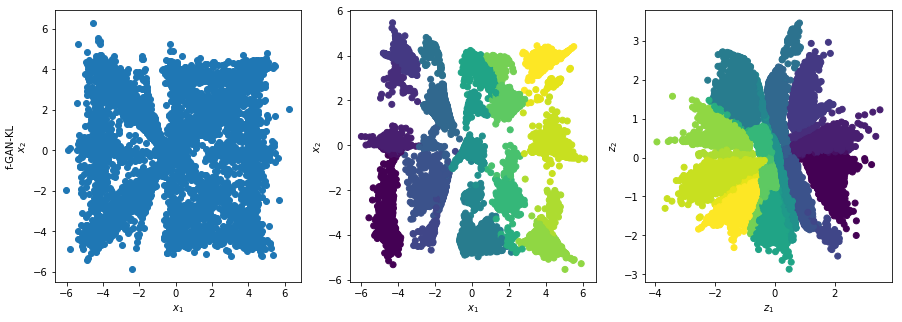}}
\subfigure{
\includegraphics[width=0.48\textwidth]{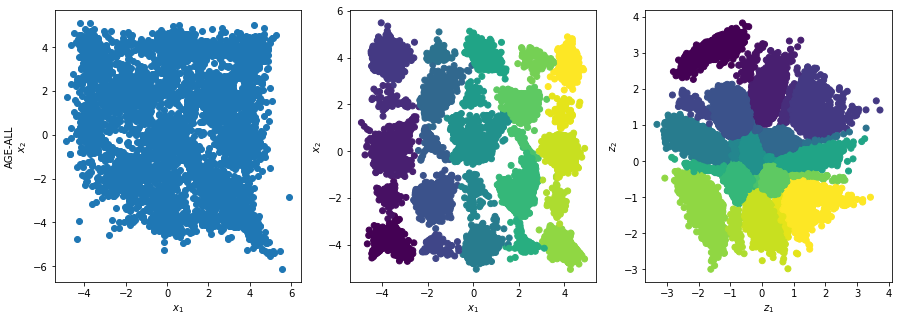}}
\vskip -0.17in
\subfigure{
\includegraphics[width=0.48\textwidth]{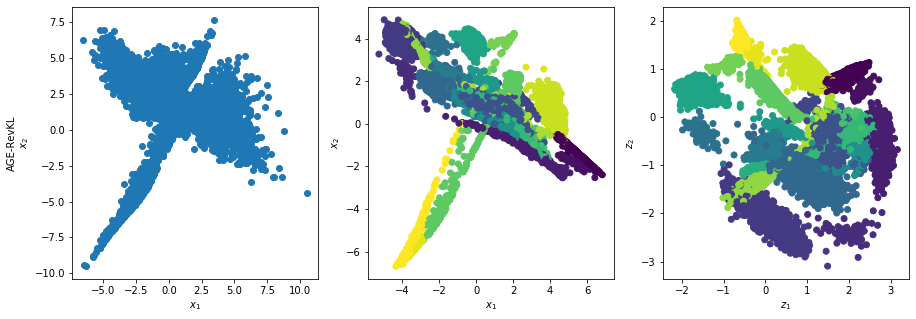}}
\subfigure{
\includegraphics[width=0.48\textwidth]{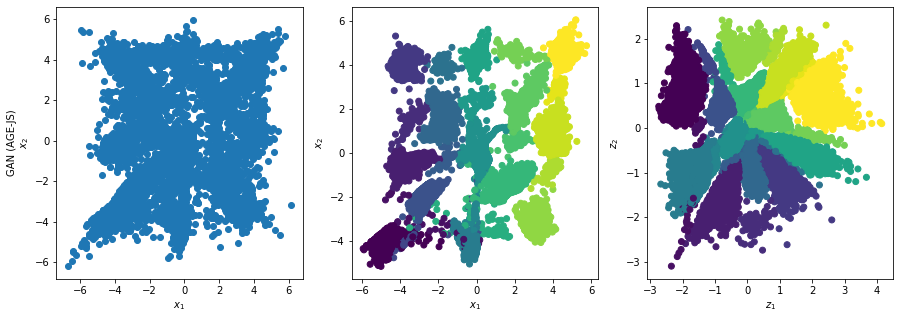}}
\vskip -0.17in
\subfigure{
\includegraphics[width=0.48\textwidth]{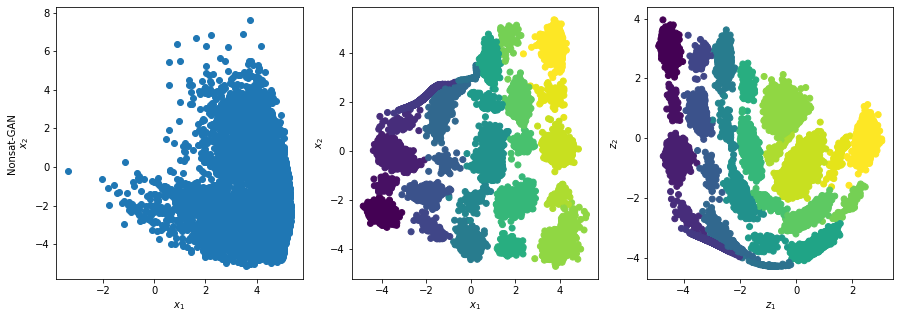}}\vspace{-0.4cm}
\subfigure{
\includegraphics[width=0.48\textwidth]{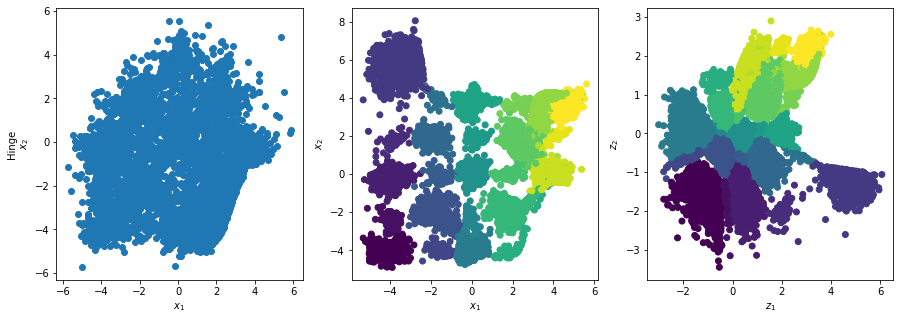}}
\caption{\small Generations, reconstructions and latent space from bidirectional generative models on the MoG dataset using various divergences as the objective. Shown top to bottom, left to right are AGES-KL, AGES-$H^2$, $f$-GAN-KL, AGES-ALL, AGES-RevKL, GAN (AGES-JS), logD-GAN, Hinge. We can clearly observe that mode collapse occurs for all divergences except KL. Moreover, the encoder learned by AGES-KL matches the aggregated posterior $p_e(z)$ and prior $p_z(z)$ the best. Apart from better mode covering, another reason for this is due to the justification of our formulation in unidirectional generative modeling.}
\vskip -.15in
\end{figure}

\begin{figure}[H]
\centering
\subfigure[Real]{
\includegraphics[width=.18\textwidth]{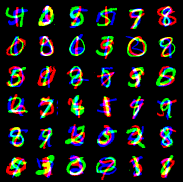}}
\subfigure[$f$-GAN-KL]{
\includegraphics[width=.18\textwidth]{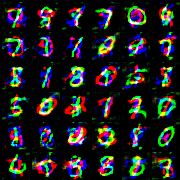}}
\subfigure[GAN]{
\includegraphics[width=.18\textwidth]{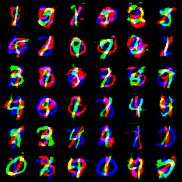}}
\subfigure[logD-GAN]{
\includegraphics[width=.18\textwidth]{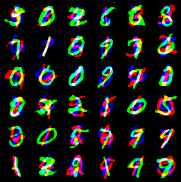}}
\subfigure[Hinge]{
\includegraphics[width=.18\textwidth]{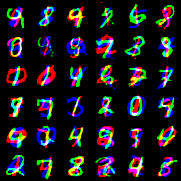}}
\vskip -0.05in
\subfigure[Generations by AGES-ALL]{
\includegraphics[width=.96\textwidth]{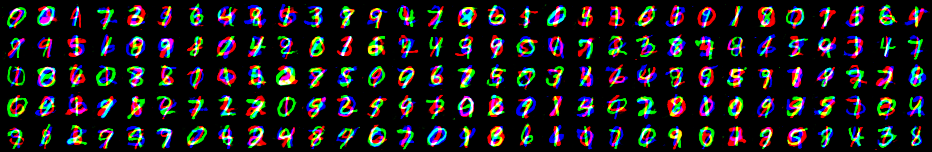}}
\vskip -0.15in
\caption{Generations on Stacked MNIST by the BGM trained using various methods.}
\end{figure}

\begin{figure}[H]
\centering
\subfigure[Real]{
\includegraphics[width=.116\textwidth]{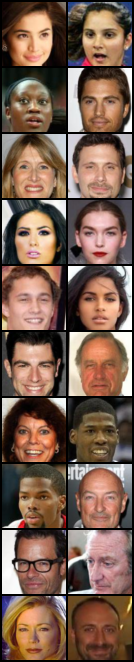}}
\subfigure[Generations by AGES-ALL]{
\includegraphics[width=.86\textwidth]{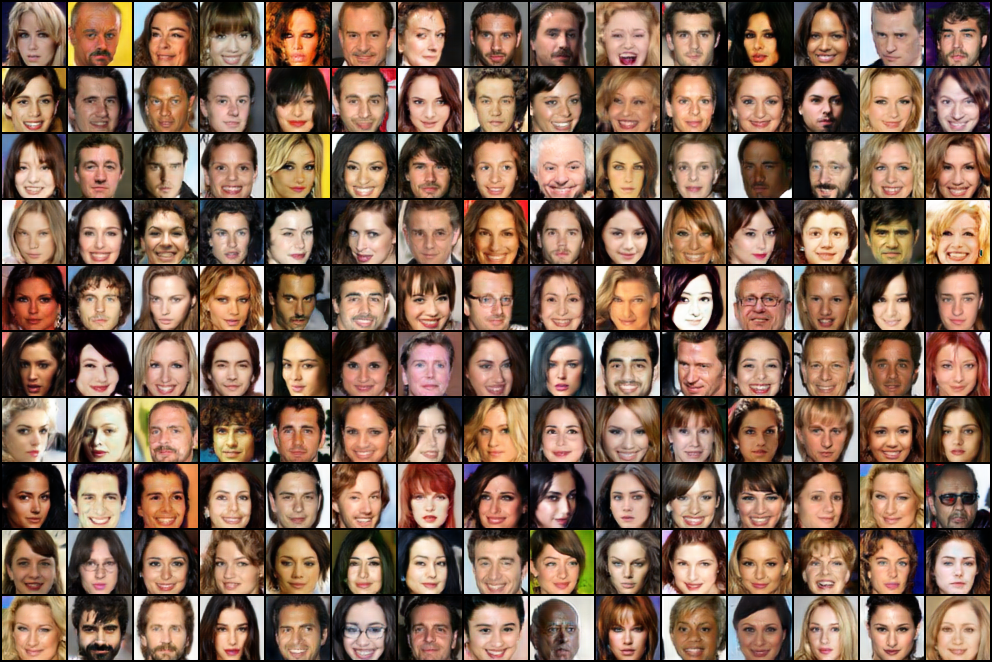}}
\vskip -0.15in
\caption{Generations on CelebA by the BGM trained using AGES.}
\end{figure}

\begin{figure}[H]
\centering
\subfigure[Real]{
\includegraphics[width=.116\textwidth]{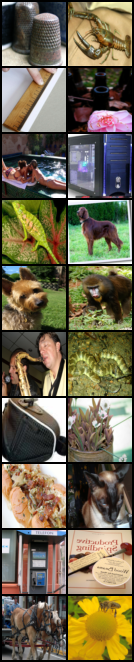}}
\subfigure[Generations by AGES-ALL]{
\includegraphics[width=.86\textwidth]{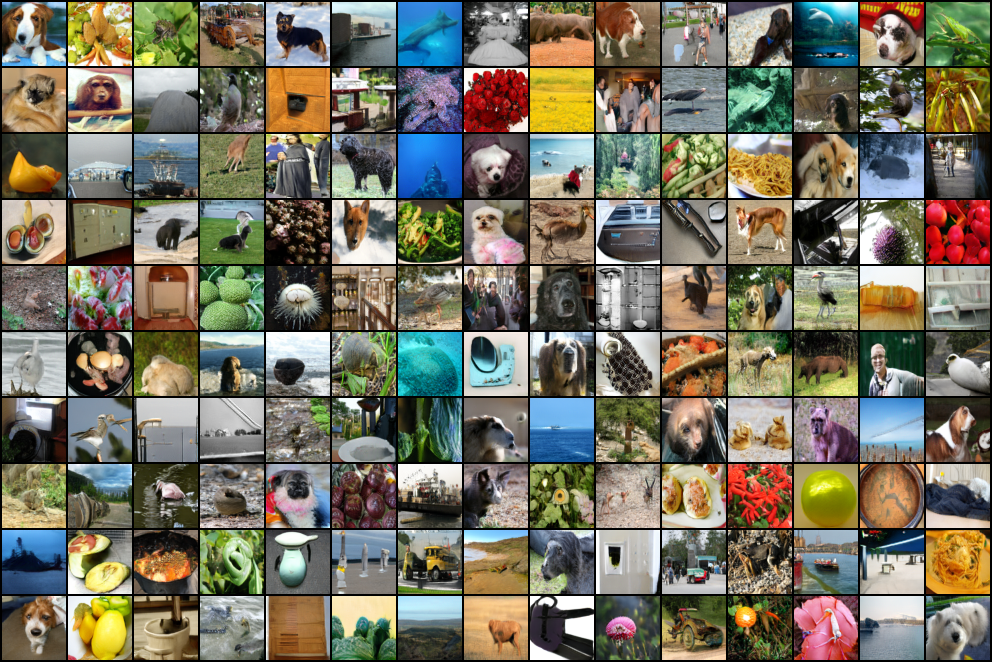}}
\vskip -0.15in
\caption{Generations on ImageNet by the BGM trained using AGES.}
\end{figure}

\begin{figure}
\centering
\includegraphics[width=.99\textwidth]{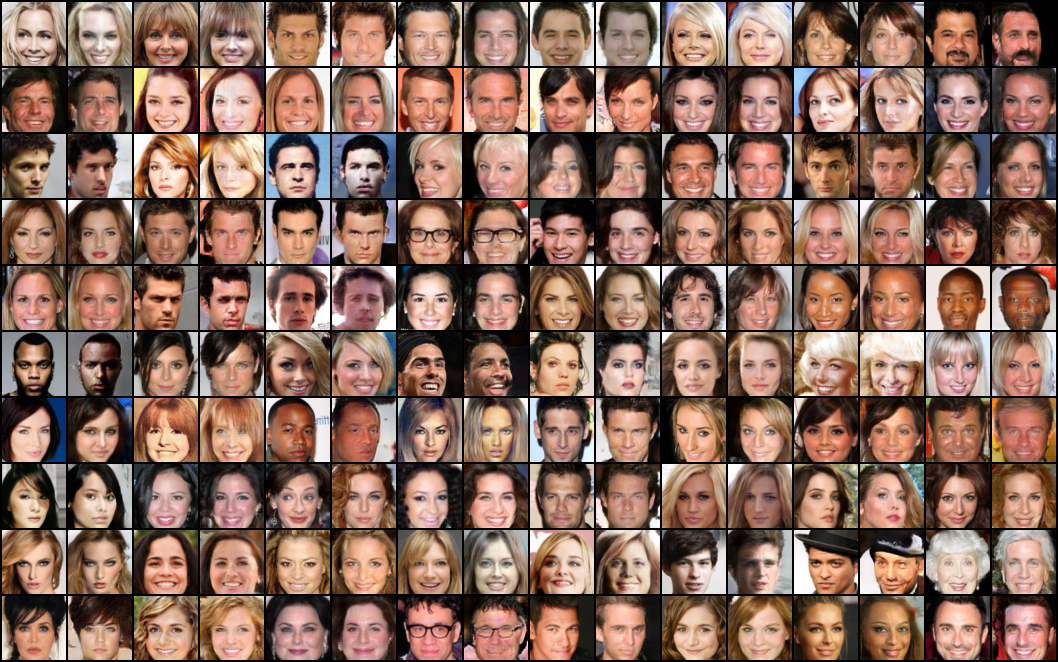}
\caption{Reconstructions on CelebA using AGES-ALL. Odd columns are real images from the validation set and even columns are the corresponding reconstructions.}
\end{figure}

\begin{figure}
\centering
\includegraphics[width=.99\textwidth]{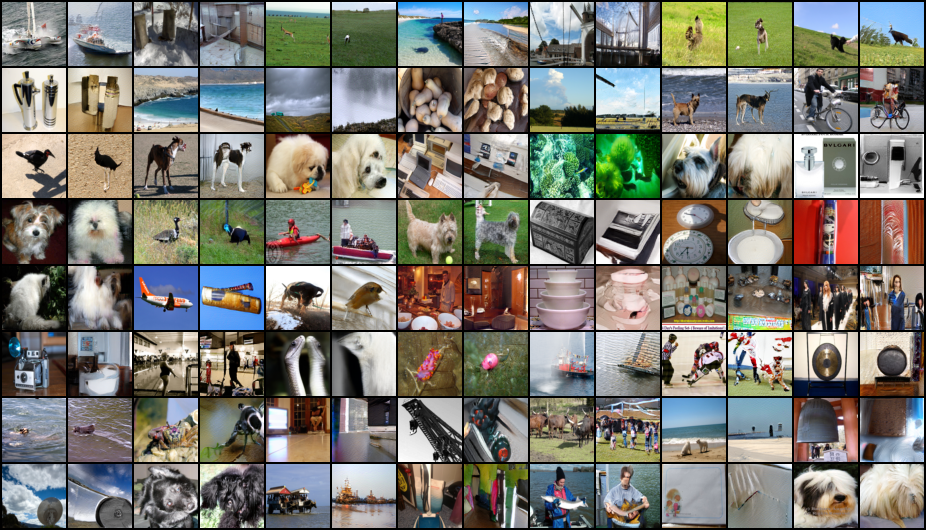}
\caption{Reconstructions on ImageNet using AGES-ALL. Odd columns are real images from the validation set and even columns are the corresponding reconstructions.}
\end{figure}

\begin{figure}
\centering
\subfigure[AGES-ALL]{
\includegraphics[width=.75\textwidth]{fig/celeba_interp.png}}
\subfigure[Hinge]{
\includegraphics[width=.75\textwidth]{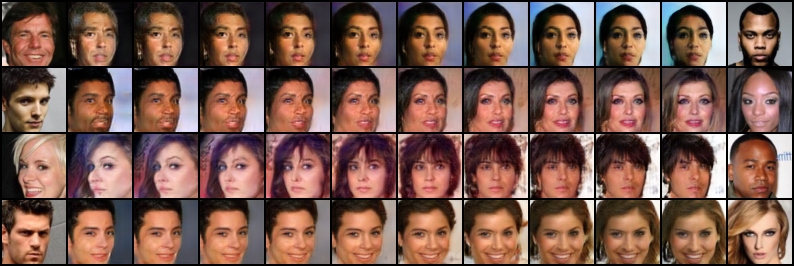}}
\subfigure[logD-GAN]{
\includegraphics[width=.75\textwidth]{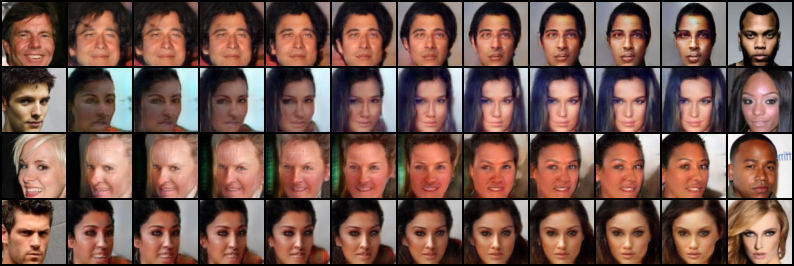}}
\subfigure[GAN]{
\includegraphics[width=.75\textwidth]{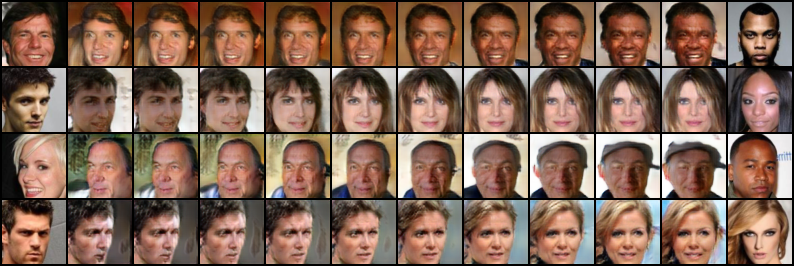}}
\caption{Latent space interpolations on CelebA validation set using various methods. The left and right columns are real images; the columns in between are generated from the latent variables interpolated linearly from the two inferred representations from the real. In contrast to other methods, AGES-ALL is able to generate smoother, more faithful and meaningful intermediate images from the interpolated latent representations between two real images.}
\end{figure}

\end{appendices}

\end{document}